\documentclass[10pt]{article}
\usepackage{fancybox}
\usepackage{amsmath}
\usepackage{mathrsfs}
\usepackage{amssymb}
\usepackage{amsthm}
\usepackage{graphicx}
\usepackage{enumitem}
\usepackage{subfigure}
\usepackage{color}
\usepackage{multirow}
\usepackage{verbatim}
\usepackage{bm}
\usepackage{bigstrut}
\usepackage[left=1in,top=1in,right=1in,bottom=1in,letterpaper]{geometry}
\usepackage[linktocpage,colorlinks,linkcolor=blue,anchorcolor=blue,citecolor=blue,urlcolor=blue]{hyperref}

\usepackage[linesnumbered,ruled]{algorithm2e}

\newtheorem{definition}{Definition}[section]
\newtheorem{theorem}{Theorem}[section]

\newtheorem{lemma}{Lemma}[section]
\newtheorem{remark}{Remark}[section]

\newcommand\m[1]{\begin{bmatrix}#1\end{bmatrix}}

\usepackage{mathtools}
\usepackage[colorlinks]{hyperref} %
\usepackage{cleveref}

\def\bE{\mathbf{E}}
\def\bR{\mathbf{R}}
\def\cL{\mathcal{L}}
\def\cC{\mathcal{C}}
\def\cO{\mathcal{O}}
\def\cS{\mathcal{S}}
\def\cB{\mathcal{B}}

\def\b1{\boldsymbol{1}}

\DeclareMathOperator*{\argmin}{arg\,min}

\newcommand{\dist}{\mathop{\textbf{dist}}}
\title{Nonconvex Stochastic Nested Optimization via Stochastic ADMM}
\author{Zhongruo Wang}

\begin{document}
\maketitle

\begin{abstract}
We consider the stochastic nested composition optimization problem where the objective is a composition of two expected-value functions. We proposed the stochastic ADMM to solve this complicated objective. In order to find an $\epsilon$-stationary point where the expected norm of the subgradient of corresponding augmented Lagrangian is smaller than $\epsilon$, the total sample complexity of our method is $\mathcal{O}(\epsilon^{-3})$ for the online case and $\cO \Bigl((2N_1 + N_2) + (2N_1 + N_2)^{1/2}\epsilon^{-2}\Bigr)$ for the finite sum case. The computational complexity is consistent with proximal version proposed in \cite{zhang2019multi}, but our algorithm can solve more general problem when the proximal mapping of the penalty is not easy to compute. 
\end{abstract}
\section{Introduction}
Consider we solve the following optimization problem:
\begin{equation}
\min_{x \in \bR^d, y \in \bR^l} F(x) + \sum_{j = 1}^m r_j(y) = \bE_{\xi_{2}} f_{2,\xi_{2}}\Bigl(\bE_{\xi_1} f_{1,\xi _1}(x)\Bigr) + \sum_{j = 1}^m r_j(y_j) \mbox{ s.t. } Ax + \sum_{j = 1}^m B_jy_j = c
\label{objective_function}
\end{equation}

An interesting special case is when $\xi_1,\xi_2$ follows a uniform distribution:
\begin{equation}
    \min_{x \in \bR^d, y \in \bR^l} F(x) + \sum_{j = 1}^m r_j(y) = \frac{1}{N_2}\sum_{j = 1}^{N_2} f_{2,j}\Bigl(\frac{1}{N_1}\sum_{j = 1}^{N_1} f_{1,j}(x)\Bigr) + \sum_{j = 1}^m r_j(y_j) \mbox{ s.t. } Ax + \sum_{j = 1}^m B_jy_j = c
\end{equation}
\section{Motivation and Previous Works}

     When penalty is not simple as $\ell_1$ penalty, for example graph guided lasso and fussed lasso, 
    we can't use simple proximal algorithms. Thus, perform operator splitting and using ADMM will be suitable for those kind of problems; ADMM for general convex and strongly convex cases has been studied in \cite{yu2017fast}. In their fomulation, they assume a very special case on the penalty that $Ax + By =0$ which is not quite general for most ADMM problems. Using ADMM to solve the same composite nonconvex composite nested objective hasn't been well studied; different variance reduced stochastic proximal methods have been studied in both convex and nonconvex cases. Proximal  version of the algorithms have also been studied for formulations of multiple level composite functions: \cite{zhang2019multi},\cite{lin2018improved}, different iteration complexity and stochastic oracle has been analyzed.  
\section{Contribution}
In this work we will present a stochastic variance reduced ADMM algorithm to solve 2-level and multiple level composite stochastic problems for both finite sum and online case. We denote the sampling number to be $N$ and the augmented Lagrangian with penalty $\rho$ to be $\cL_{\rho}$. In order to achieve $\bE \|\partial L_{\rho}(x^R,y_{[m]}^R,z^R)\|_2^2 \leq \epsilon$ for a given threshold $\epsilon > 0$, for simple mini batch estimation, we can show that iteration complexity is $\mathcal{O}(\epsilon^{-2})$ and the total complexity is $\mathcal{O}(\epsilon^{-4})$ which is too costy; when using stochastic intergraded estimator like SARAH/SPIDER, we can show that the total sampling complexity is $\cO{(\epsilon^{-3})}$ for the online case and  $\cO{((2N_1+N_2) + \sqrt{(2N_1+N_2)}\epsilon^{-2})}$ for the finite sum case. \\
\section{Assumptions and Notations}
The following assumptions are made for the further analysis of the algorithms:
\begin{enumerate}
\item $A \in \bR^{p \times d}, B_j \in \bR^{p \times l} \mbox{  }\forall j, c \in \bR^p$. 
\item $A$ and $B$ has full column rank or full row rank.
\item $F(x)$ is $L_F$-smooth
\item $f_1:\bR^d \rightarrow \bR^l$ and $f_2:\bR^l \rightarrow \bR$ are two smooth vector mapping, and each realization of the random mapping $f_{i,\xi_i}$ is $\ell_1$-Lipschitz continuous and its Jacobian $f_{i,\xi_i}'$ are $L_i$-Lipschitz continuous.
\item $\bE_{\xi_1}\|f_{1,\xi_1}(x)-f_1(x)\|_2^2 \leq \delta^2$ for all $x \in \textbf{dom} F(x)$
\item $\bE_{\xi_i}\|\nabla f_{i,\xi_i} (x) - \nabla f_i(x)\|_2^2 \leq \sigma_i^2$ for all $i$
\item $r(x)$ is a convex regularizer such as $\|\cdot\|_1$, $\|\cdot\|_2$
\end{enumerate}

\begin{itemize}
    \item $\sigma^A_{\min}$ and $\sigma^A_{\max}$ denotes the smallest and largest eigenvalue of the matrix $A^TA$, $\sigma_{\min}(H_j)$ and $\sigma_{\max}(H_j)$ denotes the smallest and largest eigenvalue of $H_j^TH_j$ for all $j \in [m]$. 
\end{itemize}
\begin{definition}
For any $\epsilon > 0$, the point $(x^*,y^*,\lambda^*)$ is said to be an $\epsilon$ stationary point of the nonconvex problem \eqref{objective_function} if it holds that:
\begin{equation}
    \begin{cases}
    \bE\|Ax^* + By^* - c\|_2^2 \leq \epsilon^2\\
    \bE\|\nabla f(x^*) - A^T \lambda^*\|_2^2 \leq  \epsilon^2\\
    \bE\|\dist (B^T\lambda^*, \partial r(y^*))^2\|_2^2 \leq \epsilon^2
    \end{cases}
    \label{stationary_point_def}
\end{equation}
where $\dist(y_0,\partial r(y)) = \inf \{\|y_0 - z\|: z \in \partial g(y)\}$, $\partial r(y)$ denotes the subgradient of $r(y)$. If $\epsilon = 0$, the point $(x^*,y^*,\lambda^*)$ is said to be a stationary point.
\end{definition}
The above inequalities \eqref{stationary_point_def} are equivalent to $\bE\|\dist(0,\partial L (x^*,y^*,\lambda^*))\|_2^2 \leq \epsilon^2$, where:
\begin{equation}
    \partial L(x,y,\lambda) = \m{\partial L (x,y,\lambda)/\partial x\\\partial L (x,y,\lambda)/\partial y\\\partial L (x,y,\lambda)/\partial \lambda}
\end{equation}
and $L(x,y,\lambda) = f(x) + g(y) - \langle \lambda, Ax + By-c\rangle$ is the Lagrangian function of the objective function \eqref{objective_function}.
\section{Main Result}
From the perspective of all the stochastic algorithm, the goal is to estimate the gradient as good as we can. The gradient of $F(x)$ can be derived from chain rule, from which we will have:
\begin{equation}
    F'(x) = (\bE_{\xi_1}[f'_{1,\xi_1}(x)])\bE_{\xi_2} [f'_{2,\xi_2}(\bE_{\xi_1} f_{1,\xi_1}(x))]
\end{equation}

Now we want to use the abbreviation to denote the approximations:
\[
Y_1^k \approx f_1(x^k), \quad Z_1^k \approx f_1'(x^k), \quad Z_2^k \approx f_2'(Y_1^k)
\]

Then the overall estimator for the gradient $F'(x)$ is $v^t = (Z_1^t)^T Z_2^T$. To solve the problem by using stochastic ADMM, we first give the augmented Lagrangian function of the problem:
\begin{equation}
    \cL_{\rho}(x,y_{[m]},z) = F(x) + \sum_{j=1}^m g_j(y_j) - \langle z, Ax + \sum_{j=1}^m B_j y_j - c\rangle + \frac{\rho}{2}\|Ax + \sum_{j = 1}^m B_j y_j - c\|_2^2
\end{equation}
Due to the stochastic gradient of the function $F$ to update $x$, we use the approximate Lagrangian over $x_k$ with the estimated gradient $v_k$:
\begin{equation}
\begin{aligned}
    \hat{\mathcal{L}}_{\rho}(x,y_{[m]},z_k,v_k) 
    = & F(x_k) + v_k^T(x-x_k) + \frac{1}{2\eta}\|x - x_k\|_G^2 \\
    &+ \sum_{j = 1}^m g_j (y_j^{k+1}) - \langle z_k , Ax + \sum_{j=1}^m\ B_j y_j - c\rangle + \frac{\rho}{2}\|Ax + \sum_{j=1}^m B_j y_j^{k+1} - c\|_2^2
    \end{aligned}
\end{equation}

In order to avoid computing the inverse of $\frac{G}{\eta} + A^TA$, we can set $G = rI_d - \rho \eta A^TA \succeq I_d$ with $r \geq \rho \eta \sigma_{\max}^A +1$ to linearize the quadratic term $\frac{\rho}{2}\|Ax + \sum_{j = 1}^m B_j y_j - c\|_2^2$. Also, in order to compute the proximal operater for each $y_i$, we can set $H_j = \tau_j I_d - \rho B_j^T B_j \succeq I_d$ with $\tau_j \geq \rho \sigma_{\max}^{B_j} + 1 $ for all $j \in [m]$ to linearize the term: $\frac{\rho}{2} \|Ax_k + \sum_{i = 1}^{j-1}B_j y_{j} + B_j y_j + \sum_{i = j+1}^m B_i y_i\|_2^2$. The question remains now is how to find a suitable gradient estimator for the composite function.
\\

Now we are ready to define the $\epsilon$-staionary point of the solution:\\

In the following the sections, we first consider about the mini-batch estimation on the gradient, we show that ADMM still convergence by using this simple implementations after suitable choice of parameters. After that, we consider use SARAH/SPIDER to estimate the nested gradient. By comparing the sampling complexity, we can show that SARAH/SPIDER based algorithm is more efficient than traditional mini-batch based algorithm. 
\section{Simple Mini-Batch Estimator}
When facing the stochastic composite objective, one simple and straight forward xstrategy is to estimate the composite gradient by using mini batch. We denote $f_{\cB} = \frac{1}{|\cB|}\sum_{i \in \cB} f_i(x)$ to be the mini-batch estimation of a funtion $f$. Since we are computing the composite gradient, we will use mini batch strategy on computing the gradient and sampling the function value at each level. Here comes with the following algorithm.

\begin{algorithm}[H]
\SetAlgoLined
 Initialization: Initial Point \(x^0\), Batch size: \(\left(\{S,s, B_1, B_2, b_1,b_2\}\right)\), $q,\eta, \rho >  0$\\
 {
  {
}
\For{\(k=0\) to \(K-1\)}{
Randomly sample batch $\mathcal{S}^k$ of $\xi_1$ with $|\mathcal{S}^k| = s$\;
$Y^k = f_{1,\mathcal{S}_1^k}(x^k)$\\
Randomly sample batch $\mathcal{B}_1^k$ of $\xi_1$ with $|\mathcal{B}_1^k| = b_1$, and $\mathcal{B}_2^k$ with $|\mathcal{B}_2^k| = b_2$\\
$Z_1^k = f'_{1,\mathcal{B}_1^t}(x^k)$\\
$Z_2^k = f'_{2,\mathcal{B}_2^t}(Y^k)$\\
Calculated the nested gradient estimation: $v^k = (Z_1^k)^T Z_2^k$\\
$y_{j}^{k+1} = \argmin_{y_i} \{ \cL_{\rho} (x^k, y_{[j-1]}^{k+1}), y_j, y_{[j+1:m]}^k\} + \frac{1}{2}\|y_j - y_j^k\|_{H_j}^2$ \text{for all} $j \in [m]$\\
$x^{k+1} = \argmin_x \hat{\cL}_{\rho}(x,y_{[m]}^{k+1}, z^k v^k)$\\
$z^{k+1} = z^k - \rho(Ax^{k+1} - \sum_{j = 1}^m B_j y_{j}^{k+1} - c)$
}

}
 \textbf{Output}: $(x,y_{[m]},z)$ choosen uniformly random from $\{x_k, y_{[m]}^k, z_k\}_{k=1}^K$
 \caption{Stochastic Nested ADMM with simple Mini Batch estimator}
 \label{algorithm_mini_batch}
\end{algorithm}
From the algorithm we can see that even though $v_k$ is a biased estimation for the full gradient, we can still analysis on the variance of the approximation and make it small. 
Firstly, in each iteration $k$ from \cite{zhang2019multi}, we know that:
\[
\begin{aligned}
&\|v^k  - F'(x^k)\|_2^2 \leq 3\|Z_1^t\|_2^2 \Bigl(\|Z_2^k - f_2'(y^k)\|_2^2 + L_2^2\|Y_1^k - f_1(x^k)\|_2^2\Bigr) + 3\ell_2^2 \|Z_1^k - f_1'(x^k)\|_2^2
\end{aligned}
\]
By using mini batch estimator, we will have the variance on each estimator to be:
\[
\bE\|Z_2^k - f_2'(y^k)\|_2^2 \leq \frac{\sigma_2^2}{b_2}, \quad \bE\|Y_1^k - f_1(x^k)\|_2^2 \leq \frac{\delta^2}{s} \quad \bE\|Z_1^k - f_1'(x^k)\|_2^2 \leq \frac{\sigma_1^2}{b_1}
\]

Also, we can have:
\begin{equation}
\begin{aligned}
\|Z_1^k\| =\|f_{1,\xi_1}'\| + \|\frac{1}{b}\sum_{\xi \in \mathcal{B}_1^r}(f'_{1,\xi_1}(x^r) - f'_{1,\xi_1}(x^{r-1}))\| \leq \|f_{1,\xi_1}'\| + \|f'_{1,\xi_1}(x^r)\| + \|f'_{1,\xi_1}(x^{r-1})\| = 3\ell_1\\
\end{aligned}
\end{equation}

Now, the variance bound on the estimated gradient by conditioning on the batches is:
\begin{equation}
\bE\|v^k - F'(x^k)\|_2^2 \leq \underbrace{\frac{27\ell_1^2\sigma_2^2}{b_2} + \frac{27\ell_1^2L_2^2\delta^2}{s} + \frac{3\ell_2^2\sigma_1^2}{b_1}}_{\cC}
\label{mini_batch_estimator}
\end{equation}
Now, we are ready to analysis the convergence of the our proposed ADMM based on SARAH/SPIDER estimator.
\begin{lemma}[Bound on the dual variable] Given the sequence $\{x^k, y_{[m]}^k, z^k\}_{k = 1}^K$ is generated by Algorithm \eqref{algorithm_mini_batch}, we will have the bound on updating the dual variable $z^k$ to be:
\begin{equation}
    \bE\|z^{k+1} - z^k\|_2^2 \leq \frac{18\cC}{\sigma^A_{\min}} +  \frac{3\sigma^2_{\max}(G)}{\sigma_{\min}^A\eta^2}\bE\|x^{k+1} - x^k\|_2^2 +\Bigl( \frac{9L^2}{\sigma_{\min}^A} + \frac{3\sigma^2_{\max}(G)}{\sigma_{\min}^A\eta^2}\Bigr)\|x^k - x^{k-1}\|_2^2
\end{equation}
\end{lemma}
\begin{proof}
By using the optimal condition of step $18$ in the algorithm \ref{algorithm_1}, we will have:
\begin{equation}
    v_k + \frac{G}{\eta}(x^{k+1} - x^k) - A^Tz_k + \rho A^T(Ax^{k+1} + \sum_{j=1}^m B_j y_j^{k+1} - c) = 0
\end{equation}
By the updating rule on the dual variable, we will have:
\begin{equation}
    A^Tz^{k+1} = v_k + \frac{G}{\eta}(x^{k+1} - x^k)
\end{equation}

It follows that:
\begin{equation}
    z_{k+1} = (A^T)^+(v_k + \frac{G}{\eta}(x^{k+1} - x^k))
\end{equation}
where $(A^T)^+$ is the pseudoinverse of $A^T$.\\

Taking expectation conditioned on $\cS^k, \cB^k_1,\cB^k_2$:
\begin{equation}
\begin{aligned}
    &\bE\|z_{k+1} - z_k\|_2^2 \\
    = &\bE \|(A^T)^+(v_k + \frac{G}{\eta}(x^{k+1} - x^k) - v_{k-1} - \frac{G}{\eta}(x^k - x^{k-1}))\|_2^2\\
    \leq & \frac{1}{\sigma_{\min}^A}\|v_k + \frac{G}{\eta}(x^{k+1} - x^k) - v_{k-1} - \frac{G}{\eta}(x^k - x^{k-1})\|_2^2\\
    \leq & \frac{1}{\sigma_{\min}^A}\Bigl[3 \bE\|v_k - v_{k-1}\|_2^2 + \frac{3\sigma_{\max}^2 (G)}{\eta^2}\bE \|x^{k+1} - x^k\|_2^2 + \frac{3\sigma_{\max}^2(G)}{\eta^2}\|x^k - x^{k-1}\|_2^2\Bigr]
    \end{aligned}
\end{equation}

Now we want to prove the upper bound of $\bE\|v_k - v_{k-1}\|_2^2$:
\begin{equation}
\begin{aligned}
    & \bE\|v_k - v_{k-1}\|_2^2 \\
    = &\bE\|v_k - \nabla f(x^k)+ \nabla f(x^k) - \nabla f(x^{k-1}) + \nabla f(x^{k-1}) - v_{k-1}\|_2^2\\
    \leq & 3\bE\|v_k - \nabla f(x^k)\|_2^2 + 3\|\nabla f(x^k ) - \nabla f(x^{k-1})\|_2^2 + 3\bE \|v_{k-1} - \nabla f(x^{k-1})\|_2^2\\
    \leq & 6\cC  + 3L^2 \|x_{k-1} - x_k\|_2^2
    \end{aligned}
\end{equation}
Where the last inequality follows from \eqref{mini_batch_estimator}.\\

In the end, we will have the bound on updating the dual variable to be:
\begin{equation}
    \bE\|z^{k+1} - z^k\|_2^2 \leq \frac{18\cC}{\sigma^A_{\min}} +  \frac{3\sigma^2_{\max}(G)}{\sigma_{\min}^A\eta^2}\bE\|x^{k+1} - x^k\|_2^2 +\Bigl( \frac{9L^2}{\sigma_{\min}^A} + \frac{3\sigma^2_{\max}(G)}{\sigma_{\min}^A\eta^2}\Bigr)\|x^k - x^{k-1}\|_2^2
\label{dual_bound_mini_batch}
\end{equation}

\end{proof}

\begin{lemma}[Point convergence]
\label{point convergence mini batch}
\end{lemma}
\begin{proof}
By the optimal condition of step 9 in algorithm \ref{algorithm_mini_batch}, we will have:
\begin{equation}
    \begin{aligned}
    0 = &(y_j^k - y_j^{k+1})^T(\partial g_j(y_j^{k+1}) - B_j^T z_k + \rho B_j^T(Ax^k + \sum_{i=1}^j B_i y_i^{k+1} + \sum_{i = j+1}^m B_i y_i^k - c) + H_j (y_j^{k+1} - y_j^k))\\
    \leq &g_j(y_j^k) - g_j(y_j^{k+1})-(z_k)^T(B_j y_j^k - B_j y_j^{k+1}) + \rho (B_j y_j^k - B_j y_j^{k+1})^T\Bigl( Ax^k + \sum_{i=1}^j B_i y_i^{k+1} + \sum_{i = j+1}^m B_i y_i^k- c \Bigr) - \|y_j^{k+1} - y_j^k\|^2_{H_j}\\
    \\
    \leq &g_j(y_j^k) - g_j(y_j^{k+1}) - z_k^T (Ax_k + \sum_{i=1}^{j-1}B_iy_i^{k+1} + \sum_{i=j}^m B_i y_i^k - c) + z_k^T (Ax_k + \sum_{i=1}^jB_iy_i^{k+1} + \sum_{i = j+1}^m B_iy_i^k - c)\\
    &+ \frac{\rho}{2}\|Ax_k + \sum_{i=1}^{j-1}B_iy_i^{k+1} + \sum_{i=j}^m B_i y_i^k - c\|_2^2 + \frac{\rho}{2}\|Ax_k + \sum_{i=1}^jB_iy_i^{k+1} + \sum_{i = j+1}^m B_iy_i^k - c\|_2^2 \\
    &- \frac{\rho}{2}\|B_j y_j^k - B_j y_j^{k+1}\|_2^2 - \|y_j^{k+1} - y_j^k\|^2_{H_j}\\
    \\
    \leq & \cL_{\rho}(x^k, y_{j-1}^{k+1}, y_{[j:m]}^k,z_k) - \cL_{\rho}(x^k, y_{j}^{k+1}, y_{[j+1:m]}^k,z_k)-\frac{\rho}{2}\|B_j y_j^k - B_j y_j^{k+1}\|_2^2 - \|y_{j}^{k+1} - y_j^k\|^2_{H_j}\\
    \\
    \leq & \cL_{\rho}(x^k, y_{j-1}^{k+1}, y_{[j:m]}^k,z_k) - \cL_{\rho}(x^k, y_{j}^{k+1}, y_{[j+1:m]}^k,z_k) - \sigma_{\min}(H_j)\|y_j^k - y_j^{k+1}\|_2^2
    \end{aligned}
\end{equation}
Now, we will have the decrease bound on update the $y_j$ component is:
\begin{equation}
\cL_{\rho}(x^k, y_{j}^{k+1}, y_{[j+1:m]}^k,z_k) - \cL_{\rho}(x^k, y_{j-1}^{k+1}, y_{[j:m]}^k,z_k) \leq  - \sigma_{\min}(H_j)\|y_j^k - y_j^{k+1}\|_2^2
    \label{bound_1_mini_batch}
\end{equation}

Since we know that $F$ is $L_F$-smooth, we will have:
\[
F(x^{k+1}) \leq F(x^k) + \langle \nabla F(x^k), x^{k+1} - x^k\rangle + \frac{L_F}{2}\|x^{k+1} - x^k\|_2^2
\]

Now, using the optimal condition on updating the $x$ component in the algorithm, we will have
\begin{equation}
    0 = (x^k - x^{k+1})^T\Bigl(v_k + \frac{G}{\eta}(x^{k+1} - x^k) - A^Tz_k + \rho A^T(Ax^{k+1} + \sum_{j=1}^m B_j y_j^{k+1} - c)\Bigr)
\end{equation}

Combine two equation above, we will have:
\begin{equation}
    \begin{aligned}
    0 \leq &f(x^k) - f(x^{k+1}) + \nabla f(x_k)^T (x^{k+1} - x^k) + \frac{L_F}{2}\|x^{k+1} - x^k\|_2^2\\
    & + (x^k - x^{k+1})^T\Bigl(v_k + \frac{G}{\eta}(x^{k+1} - x^k) - A^Tz_k + \rho A^T(Ax^{k+1} + \sum_{j=1}^m B_j y_j^{k+1} - c)\Bigr)\\
    \\
\leq & f(x^k) - f(x^{k+1}) + \frac{L}{2}\|x^{k+1} - x^k\|_2^2 - \frac{1}{\eta}\|x^k - x^{k+1}\|_G^2 + (x^k - x^{k+1})^T(v_k - \nabla f(x^k))\\
&- (z_k)^T(Ax^k - Ax^{k+1}) + \rho(Ax^k - Ax^{k+1})^T(Ax^{k+1} + \sum_{j=1}^m B_j y_j^{k+1}-c)\\
\\
\leq & f(x^k) - f(x^{k+1}) + \frac{L_F}{2}\|x^{k+1} - x^k\|_2^2 - \frac{1}{\eta}\|x^k - x^{k+1}\|_G^2 + (x_k - x_{k+1})^T(v_k - \nabla f(x^k))\\
&- z_k^T(Ax^k + \sum_{j=1}^m B_j y_j^{k+1} - c) + z_k^T(Ax^{k+1} + \sum_{j=1}^m B_j y_j^{k+1} - c) \\
& +\frac{\rho}{2}\|Ax^k + \sum_{j=1}^m B_j y_j^{k+1} - c\|_2^2 - \frac{\rho}{2}\|Ax^{k+1} + \sum_{j=1}^m B_j y_j^{k+1} - c\|_2^2 -\frac{\rho}{2} \|Ax^k - Ax^{k+1}\|_2^2\\
\\
= & \cL_{\rho}(x^k,y_{[m]}^{k+1},z^k) - \cL_{\rho}(x^{k+1},y_{[m]}^{k+1},z^k)\\
& + \frac{L_F}{2}\|x^{k+1} - x^k\|_2^2 - \frac{1}{\eta}\|x^k - x^{k+1}\|_G^2 + (x^k - x^{k+1})^T(v_k - \nabla f(x^k)) - \frac{\rho}{2}\|Ax^k - Ax^{k+1}\|_2^2\\
\\
= & \cL_{\rho}(x^k,y_{[m]}^{k+1},z^k) - \cL_{\rho}(x^{k+1},y_{[m]}^{k+1},z^k)- \Bigl(\frac{\sigma_{\min}(G)}{\eta} + \frac{\rho\sigma_{\min}^A}{2} - \frac{L_F}{2}\Bigr)\|x^{k+1} - x^k\|_2^2 + \langle x^k - x^{k+1},v_k - \nabla f(x_k)\rangle\\
\\
= & \cL_{\rho}(x^k,y_{[m]}^{k+1},z_k) - \cL_{\rho}(x^{k+1},y_{[m]}^{k+1},z_k)- \Bigl(\frac{\sigma_{\min}(G)}{\eta} + \frac{\rho\sigma_{\min}^A}{2} - L_F\Bigr)\|x_{k+1} - x_k\|_2^2 + \frac{1}{2L_F}\|v_k - \nabla f(x^k)\|_2^2\\
\\
= & \cL_{\rho}(x^k,y_{[m]}^{k+1},z_k) - \cL_{\rho}(x^{k+1},y_{[m]}^{k+1},z_k)- \Bigl(\frac{\sigma_{\min}(G)}{\eta} + \frac{\rho\sigma_{\min}^A}{2} - L_F\Bigr)\|x^{k+1} - x^k\|_2^2 + \frac{\cC}{2L_F}\\
\\
    \end{aligned}
\end{equation}
Thus, rearranging the equations, taking expectation on the batches $\cB_1^k,\cB_2^k$ and $\cS^k$, we will have:
\begin{equation}
    \bE\cL_{\rho}(x^{k+1},y_{[m]}^{k+1},z^k) -\cL_{\rho}(x^k,y_{[m]}^{k+1},z^k) \leq  - \Bigl(\frac{\sigma_{\min}(G)}{\eta} + \frac{\rho\sigma_{\min}^A}{2} - L_F\Bigr)\bE\|x^{k+1} - x^k\|_2^2 + \frac{\cC}{2L_F}
    \label{bound_2_mini_batch}
\end{equation}
Now, using the update of $z$ in the algorithm, we will have:
\begin{equation}
\begin{aligned}
    & \cL_{\rho}(x^{k+1},y_{[m]}^{k+1},z^{k+1}) - \cL_{\rho}(x^{k+1},y_{[m]}^{k+1},z^k) \\
    = &\frac{1}{\rho}\|z^{k+1} - z^k\|_2^2\\
    = & \frac{18\cC}{\rho\sigma^A_{\min}} +  \frac{3\sigma^2_{\max}(G)}{\rho \sigma_{\min}^A\eta^2}\bE\|x^{k+1} - x^k\|_2^2 +\Bigl( \frac{9L^2}{\rho \sigma_{\min}^A} + \frac{3\sigma^2_{\max}(G)}{\rho \sigma_{\min}^A\eta^2}\Bigr)\|x^k - x^{k-1}\|_2^2
\end{aligned}
\label{bound_3_mini_batch}
\end{equation}

Now, combining \eqref{bound_1_mini_batch},\eqref{bound_2_mini_batch} and \eqref{bound_3_mini_batch}, we will have:

\begin{equation}
\begin{aligned}
   & \cL_{\rho}(x^{k+1}, y_{[m]}^{k+1}, z^{k+1}) - \cL_{\rho}(x^{k}, y_{[m]}^{k}, z^{k})\\
   \\
   \leq & -\sum_{j = 1}^m \sigma_{\min}(H_j)\|y_j^k - y_j^{k+1}\|_2^2  - \Bigl(\frac{\sigma_{\min}(G)}{\eta} + \frac{\rho\sigma_{\min}^A}{2} - L_F\Bigr)\bE\|x^{k+1} - x^k\|_2^2 + \frac{\cC}{2L_F}\\
   & + \frac{18\cC}{\rho\sigma^A_{\min}} +  \frac{3\sigma^2_{\max}(G)}{\rho \sigma_{\min}^A\eta^2}\bE\|x^{k+1} - x^k\|_2^2 +\Bigl( \frac{9L_F^2}{\rho \sigma_{\min}^A} + \frac{3\sigma^2_{\max}(G)}{\rho \sigma_{\min}^A\eta^2}\Bigr)\|x^k - x^{k-1}\|_2^2\\
   \\
   \leq & (\frac{3\sigma_{\max}^2(G)}{\rho\sigma_{\min}^A\eta^2}+\frac{9L^2}{\rho \sigma_{\min}^A}) \|x^{k-1}- x^k\|_2^2-\Bigl(\frac{\sigma_{\min}(G)}{\eta} + \frac{\rho\sigma_{\min}^A}{2} - L_F - \frac{3\sigma_{\max}^2(G)}{\rho \sigma_{\min}^A \eta^2}\Bigr)\bE\|x^{k+1} - x^k\|_2^2\\
   & +\frac{\cC}{2L_F} + \frac{18\cC}{\rho \sigma^A_{\min}}
    \end{aligned}
\end{equation}

Now we defined a useful potential function:
\begin{equation}
    R_k = \bE\mathcal{L}_{\rho}(x^k,y_{[m]}^k,z_k) + (\frac{3\sigma_{\max}^2(G)}{\rho\sigma_{\min}^A\eta^2}+\frac{9L_F^2}{\rho \sigma_{\min}^A}) \bE\|x^k-x^{k-1}\|_2^2
\end{equation}

Now we can show that
\begin{equation}
\begin{aligned}
    R_{k+1} = &\bE\mathcal{L}_{\rho}(x^{k+1},y_{[m]}^{k+1},z^{k+1}) + (\frac{3\sigma_{\max}^2(G)}{\rho\sigma_{\min}^A\eta^2}+\frac{9L^2}{\rho \sigma_{\min}^A}) \bE \|x^{k+1}- x^k\|_2^2\\
    \leq & \cL_{\rho}(x^{k},y_{[m]}^{k},z^{k}) + (\frac{3\sigma_{\max}^2(G)}{\rho\sigma_{\min}^A\eta^2}+\frac{9L_F^2}{\rho \sigma_{\min}^A}) \bE \|x^{k}- x^{k-1}\|_2^2\\
    &-\Bigl(\frac{\sigma_{\min}(G)}{\eta} + \frac{\rho\sigma_{\min}^A}{2} - L_F-(\frac{3\sigma_{\max}^2(G)}{\rho\sigma_{\min}^A\eta^2}+\frac{9L_F^2}{\rho \sigma_{\min}^A})\Bigr) \bE\|x^{k+1} - x^k\|_2^2 - \sigma_{\min}^H\sum_{j = 1}^m\|y_j^k - y_j^{k+1}\|_2^2 \\
    & + \Bigl(\frac{\cC}{2L_F} + \frac{18\cC}{\rho \sigma_{\min}^A}\Bigr)\\
    \leq & R_{k} - \Lambda \bE \|x^{k+1} - x^k\|_2^2 - \sigma_{\min}^H\sum_{j = 1}^m\|y_j^k - y_j^{k+1}\|_2^2 + \Bigl( \frac{\cC}{2L_f} + \frac{18\cC}{\rho \sigma_{\min}^A}\Bigr)\\
    \leq & R_k - \gamma \Bigl(\bE \|x^{k+1} - x^k\|_2^2 + \sum_{j = 1}^m\|y_j^k - y_j^{k+1}\|_2^2\Bigr) + \Bigl( \frac{\cC}{2L_f} + \frac{18\cC}{\rho \sigma_{\min}^A}\Bigr)
    \end{aligned}
    \label{potential_bound}
\end{equation}

In which:
\[
\Lambda = \Bigl(\frac{\sigma_{\min}(G)}{\eta} + \frac{\rho\sigma_{\min}^A}{2} - L_F-(\frac{3\sigma_{\max}^2(G)}{\rho\sigma_{\min}^A\eta^2}+\frac{9L_F^2}{\rho \sigma_{\min}^A})\Bigr) > 0, \quad \gamma = \min(\sigma_{\min}^H,\Lambda)
\]
Based on the structure of the potential function $R_k$, we want to show that $R_k$ is lower bounded.

\begin{equation}
	\begin{aligned}
		& \cL_{\rho}(x^{k+1}, y_{[m]}^{k+1}, z^{k+1}) \\
		= &F(x^{k+1}) + \sum_{i = 1}^m g_j(y_{j}^{k+1}) - \langle z^{k+1},Ax^{k+1}+ \sum_{j=1}^m B_j y_j^{k+1}-c\rangle + \frac{\rho}{2}\|Ax^{k+1} + \sum_{j=1}^m B_j y_{j}^{k+1} - c\|_2^2 \\
		\geq & F(x^{k+1}) + \sum_{i = 1}^m g_j(y_{j}^{k+1}) - \langle (A^T)^+(v_k + \frac{G}{\eta}(x^{k+1} - x^k)),Ax^{k+1}+ \sum_{j=1}^m B_j y_j^{k+1}-c\rangle + \frac{\rho}{2}\|Ax^{k+1} + \sum_{j=1}^m B_j y_{j}^{k+1} - c\|_2^2 \\
		\\
		\geq &F(x^{k+1}) + \sum_{i = 1}^m g_j(y_{j}^{k+1}) - \langle (A^T)^+(v_k - \nabla F(x^k) + \nabla F(x^k) + \frac{G}{\eta}(x^{k+1} - x^k)),Ax^{k+1}+ \sum_{j=1}^m B_j y_j^{k+1}-c\rangle \\
		&+ \frac{\rho}{2}\|Ax^{k+1} + \sum_{j=1}^m B_j y_{j}^{k+1} - c\|_2^2 \\
		\\
		\geq & F(x^{k+1}) + \sum_{i = 1}^m g_j(y_{j}^{k+1}) - \frac{2}{\sigma_{\min}^A \rho} \|v_k - \nabla F(x^k)\|_2^2 - \frac{2}{\sigma_{\min}^A\rho}\|\nabla F(x^k)\|_2^2 - \frac{2\sigma_{\max}^2(G)}{\sigma_{\min}^A \eta^2 \rho}\|x^{k+1} - x^k\|_2^2\\
		& + \frac{\rho}{8}\|Ax^{k+1} + \sum_{j=1}^m B_j y_{j}^{k+1} - c\|_2^2\\
		\\
		\geq & F(x^{k+1}) + \sum_{i = 1}^m g_j(y_{j}^{k+1}) - \frac{2}{\sigma_{\min}^A \rho}\|v_k - \nabla F(x^k)\|_2^2  -\frac{2}{\sigma_{\min}^A\rho}\|\nabla F(x^k)\|_2^2- \Bigl( \frac{9L_F^2}{\sigma_{\min}^A \rho} + \frac{3\sigma_{\max}^2(G)}{\sigma_{\min}^A \eta^2 \rho}\Bigr)\|x^{k+1} - x^k\|_2^2
	\end{aligned}
\end{equation}

In all, 
\begin{equation}
	R_{k+1} \geq \bE f(x_{k+1}) + \sum_{j=1}^m \bE g_j(x_{k+1}) - \frac{2\cC}{\sigma_{\min}^A \rho} - \frac{2L_F^2}{\sigma_{\min}^A \rho} \geq f^* + \sum_{j =1}^m g_j^* - \frac{2\cC}{\sigma_{\min}^A \rho} - \frac{2L_F^2}{\sigma_{\min}^A \rho}
\end{equation}
It follows that the potential function $R_k$ is bounded below. Let's denote the lower bound of $R_k$ is $R^*$. Now we sum up the \eqref{expected_potential} and averaging all the iterates from $0$ to $K-1$ we will have:

\begin{equation}
	\frac{1}{K}\sum_{k=0}^{K-1} \Bigl(\|x^{k+1} - x^k\|_2^2 + \sum_{j = 1}^m\|y_j^k - y_j^{k+1}\|_2^2\Bigr) \leq \frac{1}{K \gamma}\Bigl(R_0 - R_*\Bigr) + \Bigl( \frac{\cC}{2L_f\gamma} + \frac{18\cC}{\rho \sigma_{\min}^A \gamma}\Bigr) 
\end{equation}
Denote $\gamma = \max(\frac{1}{2L_f\Lambda}, \frac{18}{\rho \sigma_{\min}^A \Lambda})$. Since $\cC = \frac{27\ell_1^2\sigma_2^2}{b_2} + \frac{27\ell_1^2L_2^2\delta^2}{s} + \frac{3\ell_2^2\sigma_1^2}{b_1}$, in order to achieve $\epsilon^2$ stationary point solution, we can choose:
\[
b_2 = \frac{54 
\gamma \ell_1^2 \sigma_2^2}{\epsilon^2}, s = \frac{54 \gamma \ell_1^2 L_2^2 \sigma^2 }{\epsilon^2}, b_1 = \frac{6 \gamma \ell_2^2 \sigma^2 }{\epsilon^2}
\]
From the above analysis we can see that the order of choice of the batch size is $\mathcal{O}(\epsilon^{-2})$
\end{proof}
\begin{lemma}[Stationary Point Convergence]
Given $\eta = \frac{2\alpha \sigma _{\min}(G)}{3L}(0 < \alpha < 1)$ and $\Lambda \geq \frac{\sqrt{78}L_F\kappa_G}{4\alpha} $, After $K$ iterations for algorithm \eqref{algorithm_mini_batch}, we will have:
\begin{equation}
    \bE\|\dist(0,\partial L (x^T,y^T_{[m]},z^T))\|_2^2 \leq \mathcal{O}(\frac{1}{K}) + \mathcal{O}(\cC)
\end{equation}
with $T$ is choosen uniformly from $1$ to K.
\end{lemma}
\begin{proof}
Consider the sequence $\theta_k = \bE  \Bigl[\|x^{k+1} - x^k\|_2^2 + \|x^k - x^{k-1}\|_2^2 + \sum_{i=1}^m \|y^{k+1}_i - y_i^k\|_2^2 \Bigr]$, where $\bE$ denotes the expectation conditioned on the batch $\cS_1^k, \cS_2^k,\cB^k$\\

Consider in the update of $y_i$ component, we will ha
\begin{equation}
\begin{aligned}
&\bE [\dist(0,\partial_{y_i}\cL(x,y_{[m]},z))]_{k+1} =  \bE [\dist(0, \partial g_j(y_j^{k+1}) - B_j^T z^{k+1})]_{k+1}\\
= & \bE \|B_j^T z^k - \rho B_j^T(Ax^k + \sum_{i=1}^j B_j y_j^{k+1} + \sum_{i = j+1}^m B_i y_i^k - c) - H_j(y_j^{k+1} - y_j^k) - B_j^Tz^{k+1}\|_2^2\\
= & \bE\|\rho B_j^T A(x^{k+1} - x^{k+1}) + \rho B_j^T \sum_{i = j+1}^m B_i(y_i^{k+1} - y_i^k) - H_j(y_j^{k+1} - y_j^k)\|_2^2\\
\leq & m \rho^2 \sigma_{\max}^{B_j} \sigma_{\max}^A \bE\|x^{k+1} - x^k\|_2^2 + m \rho^2 \sigma_{\max}^{B_j} \sum_{i = j+1}^m \sigma_{\max}^{B_i}\bE\|y_i^{k+1} - y_i^k\|_2^2 + m \sigma_{\max}^2(H_j)\bE\|y_j^{k+1} - y_j^k\|_2^2\\
\leq & m(\rho^2 \sigma_{\max}^B\sigma_{\max}^A +\rho^2 (\sigma_{\max}^B)^2 + \sigma_{\max}^2(H_j))\theta_k =  \nu_1 \theta_k
\end{aligned}
\end{equation}
In the updating of the $x$-component, we will have:
\begin{equation}
\begin{aligned}
    & \bE [\dist(0,\nabla_x \cL(x,y_{[m]},z))^2]_{k+1} = \bE[\|A^T z^{k+1} - \nabla f(x^{k+1})\|_2^2]\\
    \leq & \bE \|v_k - \nabla f(x^{k+1}) - \frac{G}{\eta}(x^k - x^{k+1}))\|_2^2\\
    \leq & \bE \|v_k - \nabla f(x^k) + \nabla f(x^k) - \nabla f(x^{k+1}) - \frac{G}{\eta}(x^k - x^{k+1}))\|_2^2\\
    \leq & 3\cC + 3(L_F^2 + \frac{\sigma_{\max}^2(G)}{\eta^2})\bE\|x^k - x^{k+1}\|_2^2\\
    \leq & 3(L_F^2 + \frac{\sigma_{\max}^2(G)}{\eta^2})\theta_k + 3\cC = \nu_2 \theta_k + 3\cC
    \end{aligned}
\end{equation}

In the updating of the $z$ component, we will have:
\begin{equation}
\begin{aligned}
& \bE [\dist(0,\nabla_z \cL(x,y_{[m]},z))^2]_{k+1} \\
=& \bE\|Ax^{k+1} + \sum_{j = 1}^m B_j y_j^{k+1} - c\|_2^2\\
= &\frac{1}{\rho^2}\|z^{k+1} - z^k\|_2^2\\
\leq & \frac{18\cC}{\rho\sigma^A_{\min}} +  \frac{3\sigma^2_{\max}(G)}{\rho \sigma_{\min}^A\eta^2}\bE\|x^{k+1} - x^k\|_2^2 +\Bigl( \frac{9L^2}{\rho \sigma_{\min}^A} + \frac{3\sigma^2_{\max}(G)}{\rho \sigma_{\min}^A\eta^2}\Bigr)\|x^k - x^{k-1}\|_2^2\\
\leq & \frac{18\cC}{\rho\sigma^A_{\min}} + \Bigl( \frac{9L^2}{\rho \sigma_{\min}^A} + \frac{3\sigma^2_{\max}(G)}{\rho \sigma_{\min}^A\eta^2}\Bigr) \bE \|x^{k+1} - x^k\|_2^2\\
\leq & \frac{18\cC}{\rho\sigma^A_{\min}} + \Bigl( \frac{9L^2}{\rho \sigma_{\min}^A} + \frac{3\sigma^2_{\max}(G)}{\rho \sigma_{\min}^A\eta^2}\Bigr) \theta_k = \frac{18\cC}{\rho\sigma^A_{\min}} + \nu_3 \theta_k
\end{aligned}
\end{equation}
Since we know that:
\[
\begin{aligned}
\sum_{k=1}^{K-1} \theta_k = & \sum_{k=1}^{K-1}\bE[\|x^{k+1} - x^k\|_2^2 + \|x^k - x^{k-1}\|_2^2 + \sum_{j=1}^m\|y_j^k -y_j^{k+1}\|_2^2\\
\leq & 2\sum_{k=1}^{K-1}\Bigl(\bE[\|x^{k+1} - x^k\|_2^2 + \sum_{j=1}^m \bE\|y_j^k - y_j^{k+1}\|_2^2 \Bigr)
\end{aligned}
\]

Now, consider $T$ is chosen uniformly from ${1,2,...,K-1,K}$, we will have the following bound:
\begin{equation}
\begin{aligned}
   &  \bE \|\dist (0, \partial L(x^T, y_{[m]}^T,z^T))\|_2^2\\
    \leq & \frac{3\tilde{\nu}}{K}\sum_{k = 1}^K \theta_k + \frac{18\cC}{\rho \sigma_{\min}^A} + 3\cC\\
    \leq & \frac{6\tilde{\nu}}{K} \sum_{k=1}^{K-1}\Bigl(\bE[\|x^{k+1} - x^k\|_2^2 + \sum_{j=1}^m \bE\|y_j^k - y_j^{k+1}\|_2^2 \Bigr) + (\frac{18\cC}{\rho \sigma_{\min}^A} + 3\cC)\\
    \leq &\frac{6 \tilde{\nu}}{K\gamma} (R_0 - R_*) + \Bigl( \frac{\cC}{2L_f\gamma} + \frac{18\cC}{\rho \sigma_{\min}^A \gamma}\Bigr) + \Bigl(\frac{18\cC}{\rho \sigma_{\min}^A} + 3\cC \Bigr)
    \end{aligned}
\end{equation}
with $\tilde{\nu} = \max(\nu_1,\nu_2,\nu_3)$.
Given $\eta = \frac{2\alpha \sigma _{\min}(G)}{3L}(0 < \alpha < 1)$ and $\Lambda \geq \frac{\sqrt{78}L_F\kappa_G}{4\alpha} $, we will have:
\begin{equation}
    \bE\|\dist(0,\partial L (x^T,y^T_{[m]},z^T))\|_2^2 \leq \mathcal{O}(\frac{1}{K}) + \mathcal{O}(\cC)
\end{equation}
\end{proof}
\begin{theorem}[Total Sampling complexity] Consider we want to achieve an $\epsilon$-stationary point solution, the total iteration complexity is $\cO(\epsilon^{-2})$. In order to obtain the optimal epoch length, we choose $b_1,b_2,s \sim \mathcal{O}(\epsilon^{-2})$ to be the batch size in each iteration. In all, after $\cO(\epsilon^{-2})$ iterations, the total sample complexity is $\mathcal{O}(\epsilon^{-4})$
\end{theorem}
\begin{remark}
From the above theorem, by using mini-batch estimation,we can still get the same $\mathcal{O}(1/K)$ iteration complexity as nonconvex ADMM, but in order to achieve $\epsilon$-stationary solution, the batch size will be in the same order as the total iteeration number.
\end{remark}
\section{SARAH/SPIDER Estimator}
Based on the inefficiency and the superior performance of SARAH/SPIDER \cite{fang2018spider} based algorithm, we will introduce how to use this new technique on estimating the composite (nested) gradient, which will leads to a more efficient algorithm with lower sampling complexity when dealing with those kind of problems.\\

\begin{algorithm}[H]
\SetAlgoLined
 Initialization: \(x_0\), Batch size: \(\left(\{S,s, B_1, B_2, b_1,b_2\}\right)\), $q,\eta, \rho >  0$\\
 \For{\(k=0\) to \(K-1\)}{
  \eIf{$\text{mod}(k,q) == 0$}{
Randomly sample batch $\mathcal{S}^k$ of $\xi_1$ with $|\mathcal{S}^k| = S$\;
$Y^k = f_{1,\mathcal{S}_1}(x^k)$\\
Randomly sample batch $\mathcal{B}_1^k $ of $\xi_1$ with $|\mathcal{B}_1^k| = B_1$, and $\mathcal{B}_2^k$ with $|\mathcal{B}_2^k| = B_2$\\
$Z_1^k = f'_{1,\mathcal{B}_1^k}(x^k)$\\
$Z_2^k = f'_{2,\mathcal{B}_2^k}(Y^k)$\\
}{
Randomly sample batch $\mathcal{S}^k$ of $\xi_1$ with $|\mathcal{S}^k| = s$\;
$Y^k = Y^{k-1} + f_{1,\mathcal{S}_1}(x^k)-f_{1,\mathcal{S}_1}(x^{k-1})$\\
Randomly sample batch $\mathcal{B}_1^k $ of $\xi_1$ with $|\mathcal{B}_1^k| = b_1$, and $\mathcal{B}_2^k$ with $|\mathcal{B}_2^k| = b_2$\\
$Z_1^k = Z_1^{k-1} + f'_{1,\mathcal{B}_1^t}(x^k) - f'_{1,\mathcal{B}_1^t}(x^{k-1})$\\
$Z_2^k = Z_2^{k-1} + f'_{2,\mathcal{B}_2^t}(x^k) - f'_{2,\mathcal{B}_2^t}(x^{k-1})$
}
Calculated the nested gradient estimation: $v^k = (Z_1^k)^T Z_2^k$\\
$y_{j}^{k+1} = \argmin_{y_i} \{ \cL_{\rho} (x^k, y_{[j-1]}^{k+1}), y_j, y_{[j+1:m]}^k\} + \frac{1}{2}\|y_j - y_j^k\|_{H_j}^2$\\
$x^{k+1} = \argmin_x \hat{\cL}_{\rho}(x,y_{[m]}^{k+1}, z^k v^k)$\\
$z^{k+1} = z^k - \rho(Ax^{k+1} - \sum_{j = 1}^m B_j y_{j}^{k+1} - c)$
 }
 \textbf{Output}: $(x,y_{[m]},z)$ choosen uniformly random from $\{x_k, y_{[m]}^k, z_k\}_{k=1}^K$
 \caption{Stochastic Nested ADMM with SARAH/SPIDER estimator}
 \label{algorithm_1}
\end{algorithm}

Now we want to analysis the convergence of the algorithm. First, we want to show that under the choice of the suitable parameters, we can make sure that the gradient estimator is unbiased \cite{zhang2019multi}. Throughout the paper, we will consider $n_k = \lceil k/q\rceil$ such that $(n_k-1)q +1 \leq k \leq n_kq -1$.\\

\begin{lemma}[\cite{fang2018spider}]
    Under assumption 2, the SPIDER generates stochastic gradient $v_k$ satisfies for all $(n_k-1)q +1 \leq k \leq n_kq -1$
    \[
    \bE\|v^k - \nabla f(x^k)\|_2^2 \leq \sum_{i = (n_t-1)q}^{k-1}\frac{L^2}{|\mathcal{S}_2|}\bE\|x^{i+1} - x^i\|_2^2 + \bE\|v^{(n_k-1)q} - \nabla f(x^{(n_k-1)q})\|_2^2
    \]
\end{lemma}
Based on the SARAH/SPIDER estimator above, we can have the following upper bound on the variance of the estimation.\\

Firstly, from \cite{zhang2019multi}, we know that:
\[
\begin{aligned}
&\|v^k  - F'(x^k)\|_2^2 \leq 3\|z_1^t\|_2^2 \Bigl(\|z_2^k - f_2'(y^k)\|_2^2 + L_2^2\|y_1^k - f_1(x^k)\|_2^2\Bigr) + 3\ell_2^2 \|z_1^k - f_1'(x^k)\|_2^2
\end{aligned}
\]

Now, let's bound every term in the above inequality:
\begin{equation}
\begin{aligned}
\|z_1^k\| =\|f_{1,\xi_1}'\| + \|\frac{1}{b}\sum_{\xi \in \mathcal{B}_1^r}(f'_{1,\xi_1}(x^r) - f'_{1,\xi_1}(x^{r-1}))\| \leq \|f_{1,\xi_1}'\| + \|f'_{1,\xi_1}(x^r)\| + \|f'_{1,\xi_1}(x^{r-1})\| = 3\ell_1\\
\end{aligned}
\end{equation}
For all $(n_k-1)q \leq r \leq k$, by using the SPIDER estimator, we will have:
\begin{equation}
\begin{aligned}
    \bE \|Z_2^k - f_2'(Y^k)\|_2^2 
    \leq & \bE \|Z_2^0 - f_2'(Y_1^0)\|_2^2 + \frac{L_2^2}{b_2}\sum_{r=(n_k-1)q}^{k-1} \bE[\|x^{r+1} - x^{r}\|_2^2]
    \leq \frac{\sigma_2^2}{B_2} + \ell_1^2 \frac{L_2^2}{b_2} \sum_{r = (n_k-1)q}^{k-1}\bE \|x^{r+1} - x^r\|_2^2
    \end{aligned}
\end{equation}

\begin{equation}
\bE\|Y_1^t - f_1(x^k)\|_2^2 \leq \bE\|Y_1^0 - f_1(x^0)\|_2^2 + \frac{\ell_1^2}{s}\sum_{r=(n_k-1)q}^{k-1} \bE[\|x^{r+1} - x^{r}\|_2^2] \leq \frac{\delta_1^2}{S} + \frac{\ell_1^2}{s}\sum_{r=(n_k-1)q}^{k-1} \bE[\|x^{r+1} - x^{r}\|_2^2]
\end{equation}

\begin{equation}
    \bE\|Z_1^t - f_1'(x^k)\|_2^2 \leq \bE\|Z_1^0 - f_1'(x^0)\|_2^2 + \frac{L_1^2}{b_1}\sum_{r=(n_k-1)q}^{k-1} \bE[\|x^{r+1} - x^{r}\|_2^2]\leq \frac{\sigma_1^2}{B_1} + \frac{L_1^2}{b_1}\sum_{r=(n_k-1)q}^{k-1} \bE[\|x^{r+1} - x^{r}\|_2^2]
\end{equation}
So we have the following conclusions after combining above inequalties:
\begin{itemize}
\item For the online case:
\begin{equation}
    \|v^k - F'(x^k)\|_2^2 \leq \underbrace{\frac{27\ell_1^2\sigma_2^2}{B_2} + \frac{27\ell_1^2\delta_1^2}{S} + \frac{3\ell_2^2 \sigma_1^2}{B_1}}_{\cC_1} + \underbrace{\Bigl(\frac{27\ell_1^4L_2^2}{b_2} + \frac{27\ell_1^4}{s} + \frac{3\ell_2^2 L_1^2}{b_1 }\Bigr)}_{\cC_2} \sum_{r=(n_k-1)q}^{k-1} \bE[\|x^{r+1} - x^{r}\|_2^2]
\end{equation}
\item for the finite sum case, since we calculate the full gradient in the beginning of each episode, all the estimated variance will be vanished, we will have all the :
\begin{equation}
    \|v^k - F'(x^k)\|_2^2 \leq \underbrace{\Bigl(\frac{27\ell_1^4L_2^2}{b_2} + \frac{27\ell_1^4}{s} + \frac{3\ell_2^2 L_1^2}{b_1 }\Bigr)}_{\cC_2}\sum_{r=(n_k-1)q}^{k-1} \bE[\|x^{r+1} - x^{r}\|_2^2]
\end{equation}
\end{itemize}
\subsubsection{Finite Sum Case}
\begin{lemma}[Bound on the dual variable]
Given the sequence $\{x^k, y_{[m]}^k, z^k\}_{k = 1}^K$ is generated by Algorithm \eqref{algorithm_1}, then hte upper bound on updating the dual variable $z^k$ to be:
\begin{equation}
    \bE\|z^{k+1} - z^k\|_2^2 \leq \frac{6\cC_2}{\sigma_{\min}^A} \sum_{i = (n_k-1)q}^{k-1}\bE\|x^{i+1} - x^i\|_2^2 + \frac{3\sigma_{\max}^2(G)}{\sigma_{\min}^A\eta^2}\|x^{k+1} - x^k\|_2^2 + (\frac{3\sigma_{\max}^2(G)}{\sigma_{\min}^A\eta^2}+\frac{9L^2}{\sigma_{\min}^A})\|x^{k-1}- x^k\|_2^2
\end{equation}
\end{lemma}
\begin{proof}
By using the optimal condition of step $18$ in the algorithm \ref{algorithm_1}, we will have:
\begin{equation}
    v_k + \frac{G}{\eta}(x^{k+1} - x^k) - A^Tz_k + \rho A^T(Ax^{k+1} + \sum_{j=1}^m B_j y_j^{k+1} - c) = 0
\end{equation}
By the updating rule on the dual variable, we will have:
\begin{equation}
    A^Tz^{k+1} = v_k + \frac{G}{\eta}(x^{k+1} - x^k)
\end{equation}

It follows that:
\begin{equation}
    z_{k+1} = (A^T)^+(v_k + \frac{G}{\eta}(x^{k+1} - x^k))
\end{equation}
where $(A^T)^+$ is the pseudoinverse of $A^T$.\\

Now we will have:
\begin{equation}
\begin{aligned}
    &\bE\|z_{k+1} - z_k\|_2^2 \\
    = &\bE \|(A^T)^+(v_k + \frac{G}{\eta}(x^{k+1} - x^k) - v_{k_1} - \frac{G}{\eta}(x^k - x^{k-1}))\|_2^2\\
    \leq & \frac{1}{\sigma_{\min}^A}\|v_k + \frac{G}{\eta}(x^{k+1} - x^k) - v_{k_1} - \frac{G}{\eta}(x^k - x^{k-1})\|_2^2\\
    \leq & \frac{1}{\sigma_{\min}^A}\Bigl[3 \bE\|v_k - v_{k-1}\|_2^2 + \frac{3\sigma_{\max}^2 (G)}{\eta^2}\bE \|x^{k+1} - x^k\|_2^2 + \frac{3\sigma_{\max}^2(G)}{\eta^2}\bE\|x^k - x^{k-1}\|_2^2\Bigr]
    \end{aligned}
\end{equation}

Now we want to prove the upper bound of $\bE\|v_k - v_{k-1}\|_2^2$:
\begin{equation}
\begin{aligned}
    & \bE\|v_k - v_{k-1}\|_2^2 \\
    = &\bE\|v_k - \nabla f(x^k)+ \nabla f(x^k) - \nabla f(x^{k-1}) + \nabla f(x^{k-1}) - v_{k-1}\|_2^2\\
    \leq & 3\bE\|v_k - \nabla f(x^k)\|_2^2 + 3\bE\|\nabla f(x^k ) - \nabla f(x^{k-1})\|_2^2 + 3\bE \|v_{k-1} - \nabla f(x^{k-1})\|_2^2\\
    \leq & 6\cC_2 \sum_{i = (n_k-1)q}^{k-1}\bE\|x^{i+1} - x^i\|_2^2 + 3L^2\bE\|x_{k-1} - x_k\|_2^2
    \end{aligned}
\end{equation}
In the end, we will have the bound on updating the dual variable to be:
\begin{equation}
    \bE\|z^{k+1} - z^k\|_2^2 \leq \frac{6\cC_2}{\sigma_{\min}^A} \sum_{i = (n_k-1)q}^{k-1}\bE\|x^{i+1} - x^i\|_2^2 + \frac{3\sigma_{\max}^2(G)}{\sigma_{\min}^A\eta^2}\|x^{k+1} - x^k\|_2^2 + (\frac{3\sigma_{\max}^2(G)}{\sigma_{\min}^A\eta^2}+\frac{9L^2}{\sigma_{\min}^A})\|x^{k-1}- x^k\|_2^2
\label{dual_bound}
\end{equation}

\end{proof}
\begin{lemma}[Point convergence]
Consider the sequence $\{x^k, y^k_{[m]}, z^k\}_{k=1}^K$ is generated from algorithm \eqref{algorithm_1}. Define a potential function $R_k$ as follows:
\begin{equation}
    R_k = \mathcal{L}_{\rho}(x^k,y_{[m]}^k,z_k) + (\frac{3\sigma_{\max}^2(G)}{\rho\sigma_{\min}^A\eta^2}+\frac{9L_F^2}{\rho \sigma_{\min}^A}) \|x^{k-1}- x^k\|_2^2 + \frac{2\cC_2}{\rho \sigma_{\min}^A}\sum_{i = (n_k-1)q}^{k-1}\bE\|x^{i+1} - x^i\|_2^2
\end{equation}
Denote $R_*$ is the lower bound of $R_k$. We will have:
\begin{equation}
	\frac{1}{K}\sum_{k=0}^K \Bigl(\|x^{k+1} - x^k\|_2^2 + \sum_{j = 1}^m \|y_j^k - y_j^{k+1}\|_2^2\Bigr) \leq \frac{1}{K \gamma}\Bigl(\bE [R_0] - R_*\Bigr)
\end{equation}
by setting 
\[
\cC_{2} = \frac{L_F}{q}, b_2 = \frac{27\ell_1^4 L_2^2 q}{L_F^2}, s = \frac{27\ell_1^4q}{L_F^2}, b_1 = \frac{3\ell_2^2L_1^2 q}{L_F^2}, \rho \geq \frac{\sqrt{78}L_F \kappa_G}{4\alpha}, \eta = \frac{2\alpha\sigma_{\min}(G)}{3L}(0 \leq \alpha \leq 1)
\]
\end{lemma}
\begin{proof}
By the optimal condition of step 9 in algorithm \ref{algorithm_1}, we will have:
\begin{equation}
    \begin{aligned}
    0 = &(y_j^k - y_j^{k+1})^T(\partial g_j(y_j^{k+1}) - B_j^T z_k + \rho B_j^T(Ax^k + \sum_{i=1}^j B_i y_i^{k+1} + \sum_{i = j+1}^m B_i y_i^k - c) + H_j (y_j^{k+1} - y_j^k))\\
    \leq &g_j(y_j^k) - g_j(y_j^{k+1})-(z_k)^T(B_j y_j^k - B_j y_j^{k+1}) + \rho (B_j y_j^k - B_j y_j^{k+1})^T\Bigl( Ax^k + \sum_{i=1}^j B_i y_i^{k+1} + \sum_{i = j+1}^m B_i y_i^k- c \Bigr) - \|y_j^{k+1} - y_j^k\|^2_{H_j}\\
    \\
    \leq &g_j(y_j^k) - g_j(y_j^{k+1}) - z_k^T (Ax_k + \sum_{i=1}^{j-1}B_iy_i^{k+1} + \sum_{i=j}^m B_i y_i^k - c) + z_k^T (Ax_k + \sum_{i=1}^jB_iy_i^{k+1} + \sum_{i = j+1}^m B_iy_i^k - c)\\
    &+ \frac{\rho}{2}\|Ax_k + \sum_{i=1}^{j-1}B_iy_i^{k+1} + \sum_{i=j}^m B_i y_i^k - c\|_2^2 + \frac{\rho}{2}\|Ax_k + \sum_{i=1}^jB_iy_i^{k+1} + \sum_{i = j+1}^m B_iy_i^k - c\|_2^2 \\
    &- \frac{\rho}{2}\|B_j y_j^k - B_j y_j^{k+1}\|_2^2 - \|y_j^{k+1} - y_j^k\|^2_{H_j}\\
    \\
    \leq & \cL_{\rho}(x^k, y_{j-1}^{k+1}, y_{[j:m]}^k,z_k) - \cL_{\rho}(x^k, y_{j}^{k+1}, y_{[j+1:m]}^k,z_k)-\frac{\rho}{2}\|B_j y_j^k - B_j y_j^{k+1}\|_2^2 - \|y_{j}^{k+1} - y_j^k\|^2_{H_j}\\
    \\
    \leq & \cL_{\rho}(x^k, y_{j-1}^{k+1}, y_{[j:m]}^k,z_k) - \cL_{\rho}(x^k, y_{j}^{k+1}, y_{[j+1:m]}^k,z_k) - \sigma_{\min}(H_j)\|y_j^k - y_j^{k+1}\|_2^2
    \end{aligned}
\end{equation}
Now, we will have the decrease bound on update the $y_j$ component is:
\begin{equation}
\cL_{\rho}(x^k, y_{j}^{k+1}, y_{[j+1:m]}^k,z_k) - \cL_{\rho}(x^k, y_{j-1}^{k+1}, y_{[j:m]}^k,z_k) \leq  - \sigma_{\min}(H_j)\|y_j^k - y_j^{k+1}\|_2^2
    \label{bound_1}
\end{equation}

Since we know that $f$ is $L_F$-smooth, we will have:
\[
f(x^{k+1}) \leq f(x^k) + \langle \nabla f(x^k), x^{k+1} - x^k\rangle + \frac{L_F}{2}\|x^{k+1} - x^k\|_2^2
\]

Now, using the optimal condition on updating the $x$ component in the algorithm, we will have
\begin{equation}
    0 = (x^k - x^{k+1})^T\Bigl(v_k + \frac{G}{\eta}(x^{k+1} - x^k) - A^Tz_k + \rho A^T(Ax^{k+1} + \sum_{j=1}^m B_j y_j^{k+1} - c)\Bigr)
\end{equation}

Combine two equation above, we will have:
\begin{equation}
    \begin{aligned}
    0 \leq &f(x^k) - f(x^{k+1}) + \nabla f(x_k)^T (x^{k+1} - x^k) + \frac{L_F}{2}\|x^{k+1} - x^k\|_2^2\\
    & + (x^k - x^{k+1})^T\Bigl(v_k + \frac{G}{\eta}(x^{k+1} - x^k) - A^Tz_k + \rho A^T(Ax^{k+1} + \sum_{j=1}^m B_j y_j^{k+1} - c)\Bigr)\\
    \\
\leq & f(x^k) - f(x^{k+1}) + \frac{L}{2}\|x^{k+1} - x^k\|_2^2 - \frac{1}{\eta}\|x^k - x^{k+1}\|_G^2 + (x_k - x_{k+1})^T(v_k - \nabla f(x^k))\\
&- (z_k)^T(Ax_k - Ax_{k+1}) + \rho(Ax_k - Ax_{k+1})^T(Ax_{k+1} + \sum_{j=1}^m B_j y_j^{k+1}-c)\\
\\
\leq & f(x^k) - f(x^{k+1}) + \frac{L_F}{2}\|x^{k+1} - x^k\|_2^2 - \frac{1}{\eta}\|x^k - x^{k+1}\|_G^2 + (x_k - x_{k+1})^T(v_k - \nabla f(x^k))\\
&- z_k^T(Ax^k + \sum_{j=1}^m B_j y_j^{k+1} - c) + z_k^T(Ax^{k+1} + \sum_{j=1}^m B_j y_j^{k+1} - c) \\
& +\frac{\rho}{2}\|Ax^k + \sum_{j=1}^m B_j y_j^{k+1} - c\|_2^2 - \frac{\rho}{2}\|Ax^{k+1} + \sum_{j=1}^m B_j y_j^{k+1} - c\|_2^2 -\frac{\rho}{2} \|Ax^k - Ax^{k+1}\|_2^2\\
\\
= & \cL_{\rho}(x^k,y_{[m]}^{k+1},z_k) - \cL_{\rho}(x^{k+1},y_{[m]}^{k+1},z_k)\\
& + \frac{L_F}{2}\|x^{k+1} - x^k\|_2^2 - \frac{1}{\eta}\|x^k - x^{k+1}\|_G^2 + (x_k - x_{k+1})^T(v_k - \nabla f(x^k)) - \frac{\rho}{2}\|Ax^k - Ax^{k+1}\|_2^2\\
\\
= & \cL_{\rho}(x^k,y_{[m]}^{k+1},z_k) - \cL_{\rho}(x^{k+1},y_{[m]}^{k+1},z_k)- \Bigl(\frac{\sigma_{\min}(G)}{\eta} + \frac{\rho\sigma_{\min}^A}{2} - \frac{L_F}{2}\Bigr)\|x_{k+1} - x_k\|_2^2 + \langle x_k - x_{k+1},v_k - \nabla f(x_k)\rangle\\
\\
= & \cL_{\rho}(x^k,y_{[m]}^{k+1},z_k) - \cL_{\rho}(x^{k+1},y_{[m]}^{k+1},z_k)- \Bigl(\frac{\sigma_{\min}(G)}{\eta} + \frac{\rho\sigma_{\min}^A}{2} - L_F\Bigr)\|x_{k+1} - x_k\|_2^2 + \frac{1}{2L_F}\|v_k - \nabla f(x^k)\|_2^2\\
\\
= & \cL_{\rho}(x^k,y_{[m]}^{k+1},z_k) - \cL_{\rho}(x^{k+1},y_{[m]}^{k+1},z_k)- \Bigl(\frac{\sigma_{\min}(G)}{\eta} + \frac{\rho\sigma_{\min}^A}{2} - L_F\Bigr)\|x^{k+1} - x^k\|_2^2 + \frac{\cC_2}{2L_F}\sum_{i = (n_k - 1)q}^{k-1}\bE \|x^{i+1} - x^i\|_2^2\\
\\
    \end{aligned}
\end{equation}

\begin{equation}
    \cL_{\rho}(x^{k+1},y_{[m]}^{k+1},z^k) -\cL_{\rho}(x^k,y_{[m]}^{k+1},z^k) \leq  - \Bigl(\frac{\sigma_{\min}(G)}{\eta} + \frac{\rho\sigma_{\min}^A}{2} - L_F\Bigr)\|x^{k+1} - x^k\|_2^2 + \frac{\cC_2}{2L_F}\sum_{i = (n_k - 1)q}^{k-1}\bE \|x_{i+1} - x_i\|_2^2
    \label{bound_2}
\end{equation}
Now, using the update of $z$ in the algorithm, we will have:
\begin{equation}
\begin{aligned}
    & \cL_{\rho}(x^{k+1},y_{[m]}^{k+1},z^{k+1}) - \cL_{\rho}(x^{k+1},y_{[m]}^{k+1},z^k) \\
    = &\frac{1}{\rho}\|z^{k+1} - z^k\|_2^2\\
    = & \frac{6\cC_2}{\rho \sigma_{\min}^A} \sum_{i = (n_k-1)q}^{k-1}\bE\|x^{i+1} - x^i\|_2^2 + \frac{3\sigma_{\max}^2(G)}{\rho \sigma_{\min}^A\eta^2} \|x^{k+1} - x^k\|_2^2 + (\frac{3\sigma_{\max}^2(G)}{\rho\sigma_{\min}^A\eta^2}+\frac{9L_F^2}{\rho \sigma_{\min}^A}) \|x^{k-1}- x^k\|_2^2
\end{aligned}
\label{bound_3}
\end{equation}
Now, combining \eqref{bound_1},\eqref{bound_2} and \eqref{bound_3}, we will have:

\begin{equation}
\begin{aligned}
   & \cL_{\rho}(x^{k+1}, y_{[m]}^{k+1}, z^{k+1}) - \cL_{\rho}(x^{k}, y_{[m]}^{k}, z^{k+1})\\
   \\
   \leq & -\sum_{j = 1}^m \sigma_{\min}(H_j)\|y_j^k - y_j^{k+1}\|_2^2  - \Bigl(\frac{\sigma_{\min}(G)}{\eta} + \frac{\rho\sigma_{\min}^A}{2} - L\Bigr)\|x^{k+1} - x^k\|_2^2 + \frac{\cC_2}{2L}\sum_{i = (n_k - 1)q}^{k-1}\bE \|x^{i+1} - x^i\|_2^2\\
   & + \frac{6\cC_2}{\rho \sigma_{\min}^A} \sum_{i = (n_k-1)q}^{k-1}\bE\|x^{i+1} - x^i\|_2^2 + \frac{3\sigma_{\max}^2(G)}{\rho \sigma_{\min}^A\eta^2} \|x^{k+1} - x^k\|_2^2 + (\frac{3\sigma_{\max}^2(G)}{\rho\sigma_{\min}^A\eta^2}+\frac{9L^2}{\rho \sigma_{\min}^A}) \|x^{k-1}- x^k\|_2^2\\
   \\
   \leq & -\sum_{j = 1}^m \sigma_{\min}(H_j)\|y_j^k - y_j^{k+1}\|_2^2 + (\frac{3\sigma_{\max}^2(G)}{\rho\sigma_{\min}^A\eta^2}+\frac{9L^2}{\rho \sigma_{\min}^A}) \|x^{k-1}- x^k\|_2^2\\
   &-\Bigl(\frac{\sigma_{\min}(G)}{\eta} + \frac{\rho\sigma_{\min}^A}{2} - L - \frac{3\sigma_{\max}^2(G)}{\rho \sigma_{\min}^A \eta^2}\Bigr)\|x^{k+1} - x^k\|_2^2\\
   & +\Bigl(\frac{\cC_2}{2L} + \frac{6\cC_2}{\rho \sigma_{\min}^A}\Bigr)\sum_{i = (n_k-1)q}^{k-1}\bE\|x^{i+1} - x^i\|_2^2
    \end{aligned}
\end{equation}

Now we defined a useful potential function:
\begin{equation}
    R_k = \mathcal{L}_{\rho}(x^k,y_{[m]}^k,z_k) + (\frac{3\sigma_{\max}^2(G)}{\rho\sigma_{\min}^A\eta^2}+\frac{9L_F^2}{\rho \sigma_{\min}^A}) \|x^{k-1}- x^k\|_2^2 + \frac{2\cC_2}{\rho \sigma_{\min}^A}\sum_{i = (n_k-1)q}^{k-1}\bE\|x^{i+1} - x^i\|_2^2
\end{equation}

Now we can show that
\begin{equation}
\begin{aligned}
    R_{k+1} = &\mathcal{L}_{\rho}(x^{k+1},y_{[m]}^{k+1},z_{k+1}) + (\frac{3\sigma_{\max}^2(G)}{\rho\sigma_{\min}^A\eta^2}+\frac{9L^2}{\rho \sigma_{\min}^A}) \|x^{k+1}- x^k\|_2^2 + \frac{2\cC_2}{\rho \sigma_{\min}^A}\sum_{i = (n_k-1)q}^k \bE\|x^{i+1} - x^i\|_2^2\\
    \leq & \cL_{\rho}(x^{k},y_{[m]}^{k},z^{k}) + (\frac{3\sigma_{\max}^2(G)}{\rho\sigma_{\min}^A\eta^2}+\frac{9L_F^2}{\rho \sigma_{\min}^A}) \|x^{k}- x^{k-1}\|_2^2 + \frac{2\cC_2}{\rho \sigma_{\min}^A}\sum_{i = (n_k-1)q}^{k-1} \bE\|x^{i+1} - x^i\|_2^2\\
    &-\Bigl(\frac{\sigma_{\min}(G)}{\eta} + \frac{\rho\sigma_{\min}^A}{2} - L - \frac{6\sigma_{\max}^2(G)}{\rho \sigma_{\min}^A \eta^2} - \frac{9L_F^2}{\rho \sigma_{\min}^A} - \frac{2\cC_2}{\rho \sigma_{\min}^A}\Bigr)\|x^{k+1} - x^k\|_2^2 - \sigma_{\min}^H\sum_{j = 1}^m\|y_j^k - y_j^{k+1}\|_2^2 \\
    & + \Bigl(\frac{\cC_2}{2L_F} + \frac{6\cC_2}{\rho \sigma_{\min}^A}\Bigr)\sum_{i = (n_k-1)q}^{k-1}\bE\|x^{i+1} - x^i\|_2^2\\
    \\
    \leq & R_k -\Bigl(\frac{\sigma_{\min}(G)}{\eta} + \frac{\rho\sigma_{\min}^A}{2} - L - \frac{6\sigma_{\max}^2(G)}{\rho \sigma_{\min}^A \eta^2} - \frac{9L_F^2}{\rho \sigma_{\min}^A} - \frac{2\cC_2}{\rho \sigma_{\min}^A}\Bigr)\|x^{k+1} - x^k\|_2^2 - \sigma_{\min}^H\sum_{j = 1}^m\|y_j^k - y_j^{k+1}\|_2^2 \\
    & + \Bigl(\frac{\cC_2}{2L_F} + \frac{6\cC_2}{\rho \sigma_{\min}^A}\Bigr)\sum_{i = (n_k-1)q}^{k-1}\bE\|x^{i+1} - x^i\|_2^2
    \end{aligned}
    \label{potential_bound_finite_sum}
\end{equation}

Let $(n_k - 1)q \leq l \leq n_kq -1$, telescoping \eqref{potential_bound_finite_sum} over $k$ from $(n_k-1)q$ to $k$ and take the expectation, we will have:
\begin{equation}
    \begin{aligned}
    &\bE[R_{k+1}] \\
    \leq &\bE[R_{(n_k-1)q}] - \Bigl(\frac{\sigma_{\min}(G)}{\eta} + \frac{\rho\sigma_{\min}^A}{2} - L_F - \frac{6\sigma_{\max}^2(G)}{\rho \sigma_{\min}^A \eta^2} - \frac{9L_F^2}{\rho \sigma_{\min}^A} - \frac{2\cC_2}{\rho \sigma_{\min}^A}\Bigr)\sum_{l = (n _k-1)q}^k\|x^{l+1} - x^l\|_2^2\\
    & - \sigma_{\min}^H\sum_{l = (n_k-1)q}^k\sum_{j = 1}^m\|y_j^l - y_j^{l+1}\|_2^2 + \Bigl(\frac{\cC_2}{2L} + \frac{6\cC_2}{\rho \sigma_{\min}^A}\Bigr)\sum_{l = (n_k-1)q}^k\sum_{i = (n_k-1)q}^{k-1}\bE\|x^{i+1} - x^i\|_2^2\\
    \\ 
    \leq &\bE[R_{(n_k-1)q}] - \Bigl(\frac{\sigma_{\min}(G)}{\eta} + \frac{\rho\sigma_{\min}^A}{2} - L_F - \frac{6\sigma_{\max}^2(G)}{\rho \sigma_{\min}^A \eta^2} - \frac{9L_F^2}{\rho \sigma_{\min}^A} - \frac{2\cC_2}{\rho \sigma_{\min}^A} - \frac{\cC_2 q}{2L_F} + \frac{6 \cC_2 q}{\rho \sigma_{\min}^A}\Bigr)\sum_{l = (n _k-1)q}^k\|x^{l+1} - x^l\|_2^2\\
    & - \sigma_{\min}^H\sum_{l = (n_k-1)q}^k\sum_{j = 1}^m\|y_j^l - y_j^{l+1}\|_2^2
    \end{aligned}
\label{expected_potential}    
\end{equation}

Consider we have $\cC_2 q = L_F^2$, let's denote 
\[\Lambda = \Bigl(\underbrace{\frac{\sigma_{\min}(G)}{\eta} - \frac{3L_F}{2}}_{\Lambda_1}+ \underbrace{\frac{\rho\sigma_{\min}^A}{2} - \frac{6\sigma_{\max}^2(G)}{\rho \sigma_{\min}^A \eta^2} - \frac{9L_F^2}{\rho \sigma_{\min}^A} - \frac{2L_F^2}{\rho \sigma_{\min}^A q}}_{\Lambda_2}\Bigr)
\]

Now, choosing $0 \leq \eta \leq \frac{2\sigma_{\min}(G)}{3L}$, we will have $\Lambda_1 > 0$.\\

Further, let $\eta = \frac{2\alpha \sigma_{\min}(G)}{3L_F}(0 < \alpha < 1)$, since $b > 1 > \alpha^2, \kappa_G = \frac{\sigma_{\max}^A}{\sigma_{\min}^A} > 1$, we will have:
\begin{equation}
\begin{aligned}
		 \Lambda_2 = &\frac{\rho\sigma_{\min}^A}{2} - \frac{6\sigma_{\max}^2(G)}{\rho \sigma_{\min}^A \eta^2} - \frac{9L_F^2}{\rho \sigma_{\min}^A} - \frac{2L_F^2}{\rho \sigma_{\min}^A}\\
		 = & \frac{\rho \sigma_{\min}^A}{2} - \frac{27L_F^2\kappa_G^2}{2\sigma_{\min}^A\rho \alpha^2} - \frac{9L_F^2}{\rho \sigma_{\min}^A} - \frac{2L_F^2}{\rho \sigma_{\min}^A q}\\
		 \geq & \frac{\rho \sigma_{\min}^A}{2} - \frac{27L_F^2\kappa_G^2}{2\sigma_{\min}^A\rho \alpha^2} - \frac{9 L_F^2\kappa_G^2}{\rho \sigma_{\min}^A \alpha^2} - \frac{2 L_F^2\kappa_G^2}{\rho \sigma_{\min}^A \alpha^2}\\
		 = & \frac{\rho \sigma_{\min}^A}{2} - \frac{49L_F^2 \kappa_G^2}{2\sigma_{\min}^A \rho \alpha^2}\\
		 = & \frac{\rho \sigma_{\min}^A}{4} + \frac{\rho \sigma_{\min}^A}{4} - \frac{49L_F^2 \kappa_G^2}{2\sigma_{\min}^A \rho \alpha^2}
		\end{aligned}
\end{equation}

From the above result, we just need to choose the penalty$\rho \geq \frac{\sqrt{98}L_F \kappa_G}{\sigma_{\min}^A \alpha}$. Upon the result we have, we can argue that:
\[
\Lambda \geq \frac{\sqrt{98}L_F\kappa_G}{4\alpha}
\]

Also, by choosing $\cC_2 = L_F^2/q$, we will have:
\[
\Bigl(\frac{27\ell_1^4L_2^2}{b_2} + \frac{27\ell_1^4}{s} + \frac{3\ell_2^2 L_1^2}{b_1 }\Bigr) = \frac{L_F^2}{q
}
\]
We can have that:
\[
b_2 = \frac{27\ell_1^4 L_2^2 q}{L_F^2}, s = \frac{27\ell_1^4q}{L_F^2}, b_1 = \frac{3\ell_2^2L_1^2 q}{L_F^2}
\]

Since we know that $L_F^2 \sim \mathcal{O}(\ell_1^4+\ell_2^2)$, we can argue that $b_1,s,b_2 \sim \mathcal{O}(q)$.\\

Based on the structure of the potential function $R_k$, we want to show that $R_k$ is lower bounded.

\begin{equation}
	\begin{aligned}
		& \cL_{\rho}(x^{k+1}, y_{[m]}^{k+1}, z^{k+1}) \\
		= &f(x^{k+1}) + \sum_{i = 1}^m g_j(y_{j}^{k+1}) - \langle z^{k+1},Ax^{k+1}+ \sum_{j=1}^m B_j y_j^{k+1}-c\rangle + \frac{\rho}{2}\|Ax^{k+1} + \sum_{j=1}^m B_j y_{j}^{k+1} - c\|_2^2 \\
		\geq & f(x^{k+1}) + \sum_{i = 1}^m g_j(y_{j}^{k+1}) - \langle (A^T)^+(v_k + \frac{G}{\eta}(x^{k+1} - x^k)),Ax^{k+1}+ \sum_{j=1}^m B_j y_j^{k+1}-c\rangle + \frac{\rho}{2}\|Ax^{k+1} + \sum_{j=1}^m B_j y_{j}^{k+1} - c\|_2^2 \\
		\\
		\geq &f(x^{k+1}) + \sum_{i = 1}^m g_j(y_{j}^{k+1}) - \langle (A^T)^+(v_k - \nabla f(x^k) + \nabla f(x^k) + \frac{G}{\eta}(x^{k+1} - x^k)),Ax^{k+1}+ \sum_{j=1}^m B_j y_j^{k+1}-c\rangle \\
		&+ \frac{\rho}{2}\|Ax^{k+1} + \sum_{j=1}^m B_j y_{j}^{k+1} - c\|_2^2 \\
		\\
		\geq & f(x^{k+1}) + \sum_{i = 1}^m g_j(y_{j}^{k+1}) - \frac{2}{\sigma_{\min}^A \rho} \|v_k - \nabla f(x^k)\|_2^2 - \frac{2}{\sigma_{\min}^A}\|\nabla f(x^k)\|_2^2 - \frac{2\sigma_{\max}^2(G)}{\sigma_{\min}^A \eta^2 \rho}\|x^{k+1} - x^k\|_2^2\\
		& + \frac{\rho}{8}\|Ax^{k+1} + \sum_{j=1}^m B_j y_{j}^{k+1} - c\|_2^2\\
		\\
		\geq &f(x^{k+1}) + \sum_{i = 1}^m g_j(y_{j}^{k+1}) - \frac{2L_F^2}{\sigma_{\min}^A q \rho} \sum_{i = (n_k-1)q}^{k-1}\bE\|x^{i+1} - x^i\|_2^2 - \frac{2L_F^2}{\sigma_{\min}^A \rho} - \frac{2\sigma_{\max}^2(G)}{\sigma_{\min}^A \eta^2 \rho}\|x^{k+1} - x^k\|_2^2\\
		\geq & f(x^{k+1}) + \sum_{i = 1}^m g_j(y_{j}^{k+1}) - \frac{2L_F^2}{\sigma_{\min}^A q \rho} \sum_{i = (n_k-1)q}^{k-1}\bE\|x^{i+1} - x^i\|_2^2 - \frac{2L_F^2}{\sigma_{\min}^A \rho} - \Bigl( \frac{9L_F^2}{\sigma_{\min}^A \rho} + \frac{3\sigma_{\max}^2(G)}{\sigma_{\min}^A \eta^2 \rho}\Bigr)\|x^{k+1} - x^k\|_2^2
	\end{aligned}
\end{equation}

In all, 
\begin{equation}
	R_{k+1} \geq f(x_{k+1}) + \sum_{j=1}^m g_j(x_{k+1}) - \frac{2L_F^2}{\sigma_{\min}^A \rho} \geq f^* + \sum_{j =1}^m g_j^* -  \frac{2L_F^2}{\sigma_{\min}^A \rho}
\end{equation}
It follows that the potential function $R_k$ is bounded below. Let's denote the lower bound of $R_k$ is $R^*$. Now we sum up the \eqref{expected_potential} over all the iterates from $0$ to $K$, we will have:
\begin{equation}
	\bE[R_k] - \bE[R_0] \leq - \sum_{i=0}^{K-1}(\Lambda \|x^{i+1} - x^i\|_2^2 + \sigma_{\min}^H \sum_{j = 1}^m \|y_j^i - y_j^{i+1}\|_2^2)
\end{equation}

Finally, we will have the iteration bound to be:
\begin{equation}
	\frac{1}{K}\sum_{k=0}^K \Bigl(\|x^{k+1} - x^k\|_2^2 + \sum_{j = 1}^m \|y_j^k - y_j^{k+1}\|_2^2\Bigr) \leq \frac{1}{K \gamma}\Bigl(R_0 - R_*\Bigr)
\end{equation}
In which $\gamma = \min(\Lambda, \sigma_{\min}^H)$.
\end{proof}
\begin{lemma}[Stationary point convergence]
Suppose the sequence $\{x^k, y^k_{[m]}, z^k\}$ is generated from Algorithm \eqref{algorithm_1}, there exists a constant $\tilde{\nu}$ such that, with $T$ sampling uniformly from $1,...,K$, we will have:
\begin{equation}
\bE \|\dist (0, \partial L(x^T, y_{[m]}^T,z^T))\|_2^2
    \leq \frac{9 \tilde{\nu}}{K\gamma} (R_0 - R_*)
\end{equation}
\end{lemma}
\begin{proof}
Consider the sequence $\theta_k = \bE[\|x^{k+1} - x^k\|_2^2 + \|x^k - x^{k-1}\|_2^2 + \frac{1}{q}\sum_{i = (n_k-1)q}^k\|x^{i+1} - x^i\|_2^2 + \sum_{j=1}^m\|y_j^k -y_j^{k+1}\|_2^2]$.\\

Consider in the update of $y_i$ component, we will have:
\begin{equation}
\begin{aligned}
&\bE [\dist(0,\partial_{y_i}\cL(x,y_{[m]},z))]_{k+1} =  \bE [\dist(0, \partial g_j(y_j^{k+1}) - B_j^T z^{k+1})]_{k+1}\\
= & \bE \|B_j^T z^k - \rho B_j^T(Ax^k + \sum_{i=1}^j B_j y_j^{k+1} + \sum_{i = j+1}^m B_i y_i^k - c) - H_j(y_j^{k+1} - y_j^k) - B_j^Tz^{k+1}\|_2^2\\
= & \bE\|\rho B_j^T A(x^{k+1} - x^{k+1}) + \rho B_j^T \sum_{i = j+1}^m B_i(y_i^{k+1} - y_i^k) - H_j(y_j^{k+1} - y_j^k)\|_2^2\\
\leq & m \rho^2 \sigma_{\max}^{B_j} \sigma_{\max}^A \bE\|x^{k+1} - x^k\|_2^2 + m \rho^2 \sigma_{\max}^{B_j} \sum_{i = j+1}^m \sigma_{\max}^{B_i}\bE\|y_i^{k+1} - y_i^k\|_2^2 + m \sigma_{\max}^2(H_j)\bE\|y_j^{k+1} - y_j^k\|_2^2\\
\leq & m(\rho^2 \sigma_{\max}^B\sigma_{\max}^A +\rho^2 (\sigma_{\max}^B)^2 + \sigma_{\max}^2(H_j))\theta_k =  \nu_1 \theta_k
\end{aligned}
\end{equation}
In the updating of the $x$-component, we will have:
\begin{equation}
\begin{aligned}
    & \bE [\dist(0,\nabla_x \cL(x,y_{[m]},z))^2]_{k+1} = \bE[\|A^T z^{k+1} - \nabla f(x^{k+1})\|_2^2]\\
    \leq & \bE \|v_k - \nabla f(x^{k+1}) - \frac{G}{\eta}(x^k - x^{k+1}))\|_2^2\\
    \leq & \bE \|v_k - \nabla f(x^k) + \nabla f(x^k) - \nabla f(x^{k+1}) - \frac{G}{\eta}(x^k - x^{k+1}))\|_2^2\\
    \leq & \sum_{i = (n_k-1)q}^{k-1} \frac{L_F^2}{2} \bE\|x^{i+1} - x^i\|_2^2 + 3(L_F^2 + \frac{\sigma_{\max}^2(G)}{\eta^2})\|x^k - x^{k+1}\|_2^2\\
    \leq & 3(L_F^2 + \frac{\sigma_{\max}^2(G)}{\eta^2})\theta_k = \nu_2 \theta_k
    \end{aligned}
\end{equation}

In the updating of the $z$ component, we will have:
\begin{equation}
\begin{aligned}
& \bE [\dist(0,\nabla_z \cL(x,y_{[m]},z))^2]_{k+1} \\
=& \bE\|Ax^{k+1} + \sum_{j = 1}^m B_j y_j^{k+1} - c\|_2^2\\
= &\frac{1}{\rho^2}\|z^{k+1} - z^k\|_2^2\\
\leq & \frac{6C_2}{\rho^2 \sigma_{\min}^A} \sum_{i = (n_k-1)q}^{k-1}\bE\|x^{i+1} - x^i\|_2^2 + \frac{3\sigma_{\max}^2(G)}{ \rho^2\sigma_{\min}^A\eta^2}\|x^{k+1} - x^k\|_2^2 + (\frac{3\sigma_{\max}^2(G)}{\rho^2\sigma_{\min}^A\eta^2}+\frac{9L^2}{\rho^2 \sigma_{\min}^A})\|x^{k-1}- x^k\|_2^2\\
\leq & (\frac{9L^2}{\rho^2 \sigma_{\min}^A} + \frac{3\sigma_{\max}^2 G }{\rho^2 \sigma_A^2 \eta^2})\theta_k = \nu_3 \theta_k
\end{aligned}
\end{equation}
Since we know that:
\[
\begin{aligned}
\sum_{k=1}^{K-1} \theta_k = & \sum_{k=1}^{K-1}\bE[\|x^{k+1} - x^k\|_2^2 + \|x^k - x^{k-1}\|_2^2 + \frac{1}{q}\sum_{i = (n_k-1)q}^k\|x^{i+1} - x^i\|_2^2 + \sum_{j=1}^m\|y_j^k -y_j^{k+1}\|_2^2\\
\leq & 3\Bigl(\sum_{k=1}^{K-1}\bE[\|x^{k+1} - x^k\|_2^2 + \sum_{k=1}^{K-1}\sum_{j=1}^m \bE\|y_j^k - y_j^{k+1}\|_2^2 \Bigr)
\end{aligned}
\]

Now, consider $T$ is chosen uniformly from ${1,2,...,K-1,K}$, we will have the following bound:
\begin{equation}
\begin{aligned}
   &  \bE \|\dist (0, \partial L(x^R, y_{[m]}^R,z^R))\|_2^2\\
    \leq & \frac{\tilde{\nu}}{K}\sum_{k = 1}^K \theta_k \leq \frac{9\tilde{\nu}}{K} \Bigl(\sum_{k=1}^{K-1}\bE[\|x^{k+1} - x^k\|_2^2 + \sum_{k=1}^{K-1}\sum_{j=1}^m \bE\|y_j^k - y_j^{k+1}\|_2^2 \Bigr)\\
    \leq &\frac{9 \tilde{\nu}}{K\gamma} (R_0 - R_*)
    \end{aligned}
\end{equation}
with $\tilde{\nu} = \max(\nu_1,\nu_2,\nu_3)$.
Given $\eta = \frac{2\alpha \sigma _{\min}(G)}{3L}(0 < \alpha < 1)$ and $\Lambda \geq \frac{\sqrt{78}L_F\kappa_G}{4\alpha} $, we will have:
\begin{equation}
    \bE\|\dist(0,\partial L (x^T,y^T_{[m]},z^T))\|_2^2 \leq \mathcal{O}(\frac{1}{K})
\end{equation}
\end{proof}

\begin{theorem}[Total Sampling complexity] Consider we want to achieve an $\epsilon$-stationary point solution, the total iteration complexity is $\cO(\epsilon^{-2})$. In order to obtain the optimal epoch length, we choose $q = (2N_1+N_2)^{\frac{1}{2}}$ to be the size of the inner loop and $b_1,b_2,s\sim \mathcal{O}((2N_1+N_2)^{\frac{1}{2}})$ accordingly. After $\cO(\epsilon^{-2})$ iterations, the total sample complexity is $\cO \Bigl((2N_1 + N_2) + (2N_1+N_2)^{\frac{1}{2}}\epsilon^{-2}\Bigr)$
\end{theorem}

\subsubsection{Online Case}
\begin{lemma}[Bound on the dual variable]
Given the sequence $\{x^k, y_{[m]}^k, z^k\}_{k = 1}^K$ is generated by Algorithm \eqref{algorithm_1}, we will have the bound on updating the dual variable $z^k$ to be:
\begin{equation}
    \bE\|z^{k+1} - z^k\|_2^2 \leq \frac{6\cC_1}{\sigma_{\min}^A} + \frac{6\cC_2}{\sigma_{\min}^A} \sum_{i = (n_k-1)q}^{k-1}\bE\|x^{i+1} - x^i\|_2^2 + \frac{3\sigma_{\max}^2(G)}{\sigma_{\min}^A\eta^2}\|x^{k+1} - x^k\|_2^2 + (\frac{3\sigma_{\max}^2(G)}{\sigma_{\min}^A\eta^2}+\frac{9L^2}{\sigma_{\min}^A})\|x^{k-1}- x^k\|_2^2
\end{equation}
\end{lemma}
\begin{proof}
By using the proof strategy in equation \cref{dual_bound}, we will have:
\begin{equation}
    \bE\|z^{k+1} - z^k\|_2^2 \leq \frac{6\cC_1}{\sigma_{\min}^A} + \frac{6\cC_2}{\sigma_{\min}^A} \sum_{i = (n_k-1)q}^{k-1}\bE\|x^{i+1} - x^i\|_2^2 + \frac{3\sigma_{\max}^2(G)}{\sigma_{\min}^A\eta^2}\|x^{k+1} - x^k\|_2^2 + (\frac{3\sigma_{\max}^2(G)}{\sigma_{\min}^A\eta^2}+\frac{9L^2}{\sigma_{\min}^A})\|x^{k-1}- x^k\|_2^2
\end{equation}
\end{proof}

\begin{lemma}[Point Convergence]
Consider the sequence $\{x^k, y^k_{[m]}, z^k\}_{k=1}^K$ is generated from algorithm \eqref{algorithm_1}. Define a potential function $R_k$ as follows:
\begin{equation}
    R_k = \mathcal{L}_{\rho}(x^k,y_{[m]}^k,z_k) + (\frac{3\sigma_{\max}^2(G)}{\rho\sigma_{\min}^A\eta^2}+\frac{9L_F^2}{\rho \sigma_{\min}^A}) \|x^{k-1}- x^k\|_2^2 + \frac{2\cC_2}{\rho \sigma_{\min}^A}\sum_{i = (n_k-1)q}^{k-1}\bE\|x^{i+1} - x^i\|_2^2
\end{equation}
Denote $R_*$ is the lower bound of $R_k$. We will have:
\begin{equation}
	\frac{1}{K}\sum_{k=0}^K \Bigl(\|x^{k+1} - x^k\|_2^2 + \sum_{j = 1}^m \|y_j^k - y_j^{k+1}\|_2^2\Bigr) \leq \frac{1}{K \gamma}\Bigl(\bE [R_0] - R_*\Bigr) + \Bigl(\frac{\cC_1}{2L_F}+\frac{6\cC_1}{\rho \sigma_{\min}^A}\Bigr)
\end{equation}
by setting 
\[
\cC_{2} = \frac{L_F}{q}, b_2 = \frac{27\ell_1^4 L_2^2 q}{L_F^2}, s = \frac{27\ell_1^4q}{L_F^2}, b_1 = \frac{3\ell_2^2L_1^2 q}{L_F^2}, \rho \geq \frac{\sqrt{78}L_F \kappa_G}{4\alpha}, \eta = \frac{2\alpha\sigma_{\min}(G)}{3L}(0 \leq \alpha \leq 1)
\]
\end{lemma}
\begin{proof}
By the optimal condition of step 9 in algorithm \ref{algorithm_1}, we will have:
\begin{equation}
    \begin{aligned}
    0 = &(y_j^k - y_j^{k+1})^T(\partial g_j(y_j^{k+1}) - B_j^T z_k + \rho B_j^T(Ax^k + \sum_{i=1}^j B_i y_i^{k+1} + \sum_{i = j+1}^m B_i y_i^k - c) + H_j (y_j^{k+1} - y_j^k))\\
    \leq &g_j(y_j^k) - g_j(y_j^{k+1})-(z_k)^T(B_j y_j^k - B_j y_j^{k+1}) + \rho (B_j y_j^k - B_j y_j^{k+1})^T\Bigl( Ax^k + \sum_{i=1}^j B_i y_i^{k+1} + \sum_{i = j+1}^m B_i y_i^k- c \Bigr) - \|y_j^{k+1} - y_j^k\|^2_{H_j}\\
    \\
    \leq &g_j(y_j^k) - g_j(y_j^{k+1}) - z_k^T (Ax_k + \sum_{i=1}^{j-1}B_iy_i^{k+1} + \sum_{i=j}^m B_i y_i^k - c) + z_k^T (Ax_k + \sum_{i=1}^jB_iy_i^{k+1} + \sum_{i = j+1}^m B_iy_i^k - c)\\
    &+ \frac{\rho}{2}\|Ax_k + \sum_{i=1}^{j-1}B_iy_i^{k+1} + \sum_{i=j}^m B_i y_i^k - c\|_2^2 + \frac{\rho}{2}\|Ax_k + \sum_{i=1}^jB_iy_i^{k+1} + \sum_{i = j+1}^m B_iy_i^k - c\|_2^2 \\
    &- \frac{\rho}{2}\|B_j y_j^k - B_j y_j^{k+1}\|_2^2 - \|y_j^{k+1} - y_j^k\|^2_{H_j}\\
    \\
    \leq & \cL_{\rho}(x^k, y_{j-1}^{k+1}, y_{[j:m]}^k,z_k) - \cL_{\rho}(x^k, y_{j}^{k+1}, y_{[j+1:m]}^k,z_k)-\frac{\rho}{2}\|B_j y_j^k - B_j y_j^{k+1}\|_2^2 - \|y_{j}^{k+1} - y_j^k\|^2_{H_j}\\
    \\
    \leq & \cL_{\rho}(x^k, y_{j-1}^{k+1}, y_{[j:m]}^k,z_k) - \cL_{\rho}(x^k, y_{j}^{k+1}, y_{[j+1:m]}^k,z_k) - \sigma_{\min}(H_j)\|y_j^k - y_j^{k+1}\|_2^2
    \end{aligned}
\end{equation}
Now, we will have the decrease bound on update the $y_j$ component is:
\begin{equation}
\cL_{\rho}(x^k, y_{j}^{k+1}, y_{[j+1:m]}^k,z_k) - \cL_{\rho}(x^k, y_{j-1}^{k+1}, y_{[j:m]}^k,z_k) \leq  - \sigma_{\min}(H_j)\|y_j^k - y_j^{k+1}\|_2^2
    \label{online_bound_1}
\end{equation}

Since we know that $F$ is $L_F$-smooth, we will have:
\[
f(x^{k+1}) \leq f(x^k) + \langle \nabla f(x^k), x^{k+1} - x^k\rangle + \frac{L}{2}\|x^{k+1} - x^k\|_2^2
\]

Now, using the optimal condition on updating the $x$ component in the algorithm, we will have
\begin{equation}
    0 = (x^k - x^{k+1})^T\Bigl(v_k + \frac{G}{\eta}(x^{k+1} - x^k) - A^Tz_k + \rho A^T(Ax^{k+1} + \sum_{j=1}^m B_j y_j^{k+1} - c)\Bigr)
\end{equation}

Combine two equation above, we will have:
\begin{equation}
    \begin{aligned}
    0 \leq &f(x^k) - f(x^{k+1}) + \nabla f(x_k)^T (x^{k+1} - x^k) + \frac{L_F}{2}\|x^{k+1} - x^k\|_2^2\\
    & + (x^k - x^{k+1})^T\Bigl(v_k + \frac{G}{\eta}(x^{k+1} - x^k) - A^Tz_k + \rho A^T(Ax^{k+1} + \sum_{j=1}^m B_j y_j^{k+1} - c)\Bigr)\\
    \\
\leq & f(x^k) - f(x^{k+1}) + \frac{L}{2}\|x^{k+1} - x^k\|_2^2 - \frac{1}{\eta}\|x^k - x^{k+1}\|_G^2 + (x_k - x_{k+1})^T(v_k - \nabla f(x^k))\\
&- (z_k)^T(Ax_k - Ax_{k+1}) + \rho(Ax_k - Ax_{k+1})^T(Ax_{k+1} + \sum_{j=1}^m B_j y_j^{k+1}-c)\\
\\
\leq & f(x^k) - f(x^{k+1}) + \frac{L_F}{2}\|x^{k+1} - x^k\|_2^2 - \frac{1}{\eta}\|x^k - x^{k+1}\|_G^2 + (x_k - x_{k+1})^T(v_k - \nabla f(x^k))\\
&- z_k^T(Ax^k + \sum_{j=1}^m B_j y_j^{k+1} - c) + z_k^T(Ax^{k+1} + \sum_{j=1}^m B_j y_j^{k+1} - c) \\
& +\frac{\rho}{2}\|Ax^k + \sum_{j=1}^m B_j y_j^{k+1} - c\|_2^2 - \frac{\rho}{2}\|Ax^{k+1} + \sum_{j=1}^m B_j y_j^{k+1} - c\|_2^2 -\frac{\rho}{2} \|Ax^k - Ax^{k+1}\|_2^2\\
\\
= & \cL_{\rho}(x^k,y_{[m]}^{k+1},z_k) - \cL_{\rho}(x^{k+1},y_{[m]}^{k+1},z_k)\\
& + \frac{L_F}{2}\|x^{k+1} - x^k\|_2^2 - \frac{1}{\eta}\|x^k - x^{k+1}\|_G^2 + (x_k - x_{k+1})^T(v_k - \nabla f(x^k)) - \frac{\rho}{2}\|Ax^k - Ax^{k+1}\|_2^2\\
\\
= & \cL_{\rho}(x^k,y_{[m]}^{k+1},z_k) - \cL_{\rho}(x^{k+1},y_{[m]}^{k+1},z_k)- \Bigl(\frac{\sigma_{\min}(G)}{\eta} + \frac{\rho\sigma_{\min}^A}{2} - \frac{L_F}{2}\Bigr)\|x_{k+1} - x_k\|_2^2 + \langle x_k - x_{k+1},v_k - \nabla f(x_k)\rangle\\
\\
= & \cL_{\rho}(x^k,y_{[m]}^{k+1},z_k) - \cL_{\rho}(x^{k+1},y_{[m]}^{k+1},z_k)- \Bigl(\frac{\sigma_{\min}(G)}{\eta} + \frac{\rho\sigma_{\min}^A}{2} - L_F\Bigr)\|x_{k+1} - x_k\|_2^2 + \frac{1}{2L_F}\|v_k - \nabla f(x^k)\|_2^2\\
\\
= & \cL_{\rho}(x^k,y_{[m]}^{k+1},z_k) - \cL_{\rho}(x^{k+1},y_{[m]}^{k+1},z_k)- \Bigl(\frac{\sigma_{\min}(G)}{\eta} + \frac{\rho\sigma_{\min}^A}{2} - L_F\Bigr)\|x_{k+1} - x_k\|_2^2 \\
& + \frac{\cC_2}{2L_F}\sum_{i = (n_k - 1)q}^{k-1}\bE \|x_{i+1} - x_i\|_2^2 + \frac{\cC_1}{2L_F}\\
\\
    \end{aligned}
\end{equation}

\begin{equation}
\begin{aligned}
    &\cL_{\rho}(x^{k+1},y_{[m]}^{k+1},z^k) -\cL_{\rho}(x^k,y_{[m]}^{k+1},z^k) \\
    \leq  &- \Bigl(\frac{\sigma_{\min}(G)}{\eta} + \frac{\rho\sigma_{\min}^A}{2} - L_F\Bigr)\|x^{k+1} - x^k\|_2^2 + \frac{\cC_2}{2L_F}\sum_{i = (n_k - 1)q}^{k-1}\bE \|x_{i+1} - x_i\|_2^2 + \frac{\cC_1}{2L_F}
    \label{online_bound_2}
    \end{aligned}
\end{equation}
Now, using the update of $z$ in the algorithm, we will have:
\begin{equation}
\begin{aligned}
    & \cL_{\rho}(x^{k+1},y_{[m]}^{k+1},z^{k+1}) - \cL_{\rho}(x^{k+1},y_{[m]}^{k+1},z^k) \\
    = &\frac{1}{\rho}\|z^{k+1} - z^k\|_2^2\\
    = & \frac{6\cC_1}{\rho \sigma_{\min}^A}+\frac{6\cC_2}{\rho \sigma_{\min}^A} \sum_{i = (n_k-1)q}^{k-1}\bE\|x^{i+1} - x^i\|_2^2 + \frac{3\sigma_{\max}^2(G)}{\rho \sigma_{\min}^A\eta^2} \|x^{k+1} - x^k\|_2^2 + (\frac{3\sigma_{\max}^2(G)}{\rho\sigma_{\min}^A\eta^2}+\frac{9L_F^2}{\rho \sigma_{\min}^A}) \|x^{k-1}- x^k\|_2^2
\end{aligned}
\label{online_bound_3}
\end{equation}
Now, combining \eqref{online_bound_1},\eqref{online_bound_2} and \eqref{online_bound_3}, we will have:

\begin{equation}
\begin{aligned}
   & \cL_{\rho}(x^{k+1}, y_{[m]}^{k+1}, z^{k+1}) - \cL_{\rho}(x^{k}, y_{[m]}^{k}, z^{k+1})\\
   \\
   \leq & -\sum_{j = 1}^m \sigma_{\min}(H_j)\|y_j^k - y_j^{k+1}\|_2^2  - \Bigl(\frac{\sigma_{\min}(G)}{\eta} + \frac{\rho\sigma_{\min}^A}{2} - L_F\Bigr)\|x^{k+1} - x^k\|_2^2 + \frac{\cC_2}{2L_F}\sum_{i = (n_k - 1)q}^{k-1}\bE \|x^{i+1} - x^i\|_2^2 + \frac{\cC_1}{2L_F}\\
   & +\frac{6\cC_1}{\rho \sigma_{\min}^A} + \frac{6\cC_2}{\rho \sigma_{\min}^A} \sum_{i = (n_k-1)q}^{k-1}\bE\|x^{i+1} - x^i\|_2^2 + \frac{3\sigma_{\max}^2(G)}{\rho \sigma_{\min}^A\eta^2} \|x^{k+1} - x^k\|_2^2 + (\frac{3\sigma_{\max}^2(G)}{\rho\sigma_{\min}^A\eta^2}+\frac{9L^2}{\rho \sigma_{\min}^A}) \|x^{k-1}- x^k\|_2^2\\
   \\
   \leq & -\sum_{j = 1}^m \sigma_{\min}(H_j)\|y_j^k - y_j^{k+1}\|_2^2 + (\frac{3\sigma_{\max}^2(G)}{\rho\sigma_{\min}^A\eta^2}+\frac{9L^2}{\rho \sigma_{\min}^A}) \|x^{k-1}- x^k\|_2^2\\
   &-\Bigl(\frac{\sigma_{\min}(G)}{\eta} + \frac{\rho\sigma_{\min}^A}{2} - L - \frac{3\sigma_{\max}^2(G)}{\rho \sigma_{\min}^A \eta^2}\Bigr)\|x^{k+1} - x^k\|_2^2\\
   & +\Bigl(\frac{\cC_2}{2L} + \frac{6\cC_2}{\rho \sigma_{\min}^A}\Bigr)\sum_{i = (n_k-1)q}^{k-1}\bE\|x^{i+1} - x^i\|_2^2\\
   & + \frac{\cC_1}{2L_F}+\frac{6\cC_1}{\rho \sigma_{\min}^A}
    \end{aligned}
\end{equation}

Now we defined a useful potential function:
\begin{equation}
    R_k = \mathcal{L}_{\rho}(x^k,y_{[m]}^k,z_k) + (\frac{3\sigma_{\max}^2(G)}{\rho\sigma_{\min}^A\eta^2}+\frac{9L_F^2}{\rho \sigma_{\min}^A}) \|x^{k-1}- x^k\|_2^2 + \frac{2\cC_2}{\rho \sigma_{\min}^A}\sum_{i = (n_k-1)q}^{k-1}\bE\|x^{i+1} - x^i\|_2^2
\end{equation}

Now we can show that
\begin{equation}
\begin{aligned}
    R_{k+1} = &\mathcal{L}_{\rho}(x^{k+1},y_{[m]}^{k+1},z_{k+1}) + (\frac{3\sigma_{\max}^2(G)}{\rho\sigma_{\min}^A\eta^2}+\frac{9L^2}{\rho \sigma_{\min}^A}) \|x^{k+1}- x^k\|_2^2 + \frac{2\cC_2}{\rho \sigma_{\min}^A}\sum_{i = (n_k-1)q}^k \bE\|x^{i+1} - x^i\|_2^2\\
    \leq & \cL_{\rho}(x^{k},y_{[m]}^{k},z^{k}) + (\frac{3\sigma_{\max}^2(G)}{\rho\sigma_{\min}^A\eta^2}+\frac{9L_F^2}{\rho \sigma_{\min}^A}) \|x^{k}- x^{k-1}\|_2^2 + \frac{2\cC_2}{\rho \sigma_{\min}^A}\sum_{i = (n_k-1)q}^{k-1} \bE\|x^{i+1} - x^i\|_2^2\\
    &-\Bigl(\frac{\sigma_{\min}(G)}{\eta} + \frac{\rho\sigma_{\min}^A}{2} - L - \frac{6\sigma_{\max}^2(G)}{\rho \sigma_{\min}^A \eta^2} - \frac{9L_F^2}{\rho \sigma_{\min}^A} - \frac{2\cC_2}{\rho \sigma_{\min}^A}\Bigr)\|x^{k+1} - x^k\|_2^2 - \sigma_{\min}^H\sum_{j = 1}^m\|y_j^k - y_j^{k+1}\|_2^2 \\
    & + \Bigl(\frac{\cC_2}{2L_F} + \frac{6\cC_2}{\rho \sigma_{\min}^A}\Bigr)\sum_{i = (n_k-1)q}^{k-1}\bE\|x^{i+1} - x^i\|_2^2 +\frac{\cC_1}{2L_F}+\frac{6\cC_1}{\rho \sigma_{\min}^A}\\
    \\
    \leq & R_k -\Bigl(\frac{\sigma_{\min}(G)}{\eta} + \frac{\rho\sigma_{\min}^A}{2} - L - \frac{6\sigma_{\max}^2(G)}{\rho \sigma_{\min}^A \eta^2} - \frac{9L_F^2}{\rho \sigma_{\min}^A} - \frac{2\cC_2}{\rho \sigma_{\min}^A}\Bigr)\|x^{k+1} - x^k\|_2^2 - \sigma_{\min}^H\sum_{j = 1}^m\|y_j^k - y_j^{k+1}\|_2^2 \\
    & + \Bigl(\frac{6\cC_2}{\rho \sigma_{\min}^A}\Bigr)\sum_{i = (n_k-1)q}^{k-1}\bE\|x^{i+1} - x^i\|_2^2+\frac{\cC_1}{2L_F}+\frac{6\cC_1}{\rho \sigma_{\min}^A}
    \end{aligned}
    \label{potential_bound}
\end{equation}

Let $(n_k - 1)q \leq l \leq n_kq -1$, telescoping the \eqref{potential_bound} over $k$ from $(n_k-1)q$ to $k$ and take the expectation, we will have:
\begin{equation}
    \begin{aligned}
    &\bE[R_{k+1}] \\
    \leq &\bE[R_{(n_k-1)q}] - \Bigl(\frac{\sigma_{\min}(G)}{\eta} + \frac{\rho\sigma_{\min}^A}{2} - L_F - \frac{6\sigma_{\max}^2(G)}{\rho \sigma_{\min}^A \eta^2} - \frac{9L_F^2}{\rho \sigma_{\min}^A} - \frac{2\cC_2}{\rho \sigma_{\min}^A}\Bigr)\sum_{l = (n _k-1)q}^k\|x^{l+1} - x^l\|_2^2\\
    & - \sigma_{\min}^H\sum_{l = (n_k-1)q}^k\sum_{j = 1}^m\|y_j^l - y_j^{l+1}\|_2^2 + \Bigl(\frac{\cC_2}{2L} + \frac{6\cC_2}{\rho \sigma_{\min}^A}\Bigr)\sum_{l = (n_k-1)q}^k\sum_{i = (n_k-1)q}^{k-1}\bE\|x^{i+1} - x^i\|_2^2+\frac{\cC_1}{2L_F}+\frac{6\cC_1}{\rho \sigma_{\min}^A}\\
    \\ 
    \leq &\bE[R_{(n_k-1)q}] - \Bigl(\frac{\sigma_{\min}(G)}{\eta} + \frac{\rho\sigma_{\min}^A}{2} - L_F - \frac{6\sigma_{\max}^2(G)}{\rho \sigma_{\min}^A \eta^2} - \frac{9L_F^2}{\rho \sigma_{\min}^A} - \frac{2\cC_2}{\rho \sigma_{\min}^A} - \frac{\cC_2 q}{2L_F} + \frac{6 \cC_2 q}{\rho \sigma_{\min}^A}\Bigr)\sum_{l = (n _k-1)q}^k\|x^{l+1} - x^l\|_2^2\\
    & - \sigma_{\min}^H\sum_{l = (n_k-1)q}^k\sum_{j = 1}^m\|y_j^l - y_j^{l+1}\|_2^2 + \Bigl(\frac{\cC_1}{2L_F}+\frac{6\cC_1}{\rho \sigma_{\min}^A}\Bigr)
    \end{aligned}
\label{online_expected_potential}    
\end{equation}

Consider we have $\cC_2 q = L_F^2$, let's denote 
\[\Lambda = \Bigl(\underbrace{\frac{\sigma_{\min}(G)}{\eta} - \frac{3L_F}{2}}_{\Lambda_1}+ \underbrace{\frac{\rho\sigma_{\min}^A}{2} - \frac{6\sigma_{\max}^2(G)}{\rho \sigma_{\min}^A \eta^2} - \frac{9L_F^2}{\rho \sigma_{\min}^A} - \frac{2L_F^2}{\rho \sigma_{\min}^A q}}_{\Lambda_2}\Bigr)
\]

Now, choosing $0 \leq \eta \leq \frac{2\sigma_{\min}(G)}{3L}$, we will have $\Lambda_1 > 0$.\\

Further, let $\eta = \frac{2\alpha \sigma_{\min}(G)}{3L_F}(0 < \alpha < 1)$, since $b > 1 > \alpha^2, \kappa_G = \frac{\sigma_{\max}^A}{\sigma_{\min}^A} > 1$, we will have:
\begin{equation}
\begin{aligned}
		 \Lambda_2 = &\frac{\rho\sigma_{\min}^A}{2} - \frac{6\sigma_{\max}^2(G)}{\rho \sigma_{\min}^A \eta^2} - \frac{9L_F^2}{\rho \sigma_{\min}^A} - \frac{2L_F^2}{\rho \sigma_{\min}^A}\\
		 = & \frac{\rho \sigma_{\min}^A}{2} - \frac{27L_F^2\kappa_G^2}{2\sigma_{\min}^A\rho \alpha^2} - \frac{9L_F^2}{\rho \sigma_{\min}^A} - \frac{2L_F^2}{\rho \sigma_{\min}^A q}\\
		 \geq & \frac{\rho \sigma_{\min}^A}{2} - \frac{27L_F^2\kappa_G^2}{2\sigma_{\min}^A\rho \alpha^2} - \frac{9 L_F^2\kappa_G^2}{\rho \sigma_{\min}^A \alpha^2} - \frac{2 L_F^2\kappa_G^2}{\rho \sigma_{\min}^A \alpha^2}\\
		 = & \frac{\rho \sigma_{\min}^A}{2} - \frac{49L_F^2 \kappa_G^2}{2\sigma_{\min}^A \rho \alpha^2}\\
		 = & \frac{\rho \sigma_{\min}^A}{4} + \frac{\rho \sigma_{\min}^A}{4} - \frac{49L_F^2 \kappa_G^2}{2\sigma_{\min}^A \rho \alpha^2}
		\end{aligned}
\end{equation}

From the above result, we just need to choose the penalty $\rho \geq \frac{\sqrt{78}L_F \kappa_G}{\sigma_{\min}^A \alpha}$. Upon the result we have, we can argue that:
\[
\Lambda \geq \frac{\sqrt{78}L_F\kappa_G}{4\alpha}
\]

Also, by choosing $\cC_2 = L_F^2/q$, we will have:
\[
\Bigl(\frac{27\ell_1^4L_2^2}{b_2} + \frac{27\ell_1^4}{s} + \frac{3\ell_2^2 L_1^2}{b_1 }\Bigr) = \frac{L_F^2}{q
}
\]
We can have that:
\[
b_2 = \frac{27\ell_1^4 L_2^2 q}{L_F^2}, s = \frac{27\ell_1^4q}{L_F^2}, b_1 = \frac{3\ell_2^2L_1^2 q}{L_F^2}
\]

Since we know that $L_F^2 \sim \mathcal{O}(\ell_1^4+\ell_2^2)$, we can argue that $b_1,s,b_2 \sim \mathcal{O}(q)$.\\

Based on the structure of the potential function $R_k$, we want to show that $R_k$ is lower bounded.

\begin{equation}
	\begin{aligned}
		& \cL_{\rho}(x^{k+1}, y_{[m]}^{k+1}, z^{k+1}) \\
		= &f(x^{k+1}) + \sum_{i = 1}^m g_j(y_{j}^{k+1}) - \langle z^{k+1},Ax^{k+1}+ \sum_{j=1}^m B_j y_j^{k+1}-c\rangle + \frac{\rho}{2}\|Ax^{k+1} + \sum_{j=1}^m B_j y_{j}^{k+1} - c\|_2^2 \\
		\geq & f(x^{k+1}) + \sum_{i = 1}^m g_j(y_{j}^{k+1}) - \langle (A^T)^+(v_k + \frac{G}{\eta}(x^{k+1} - x^k)),Ax^{k+1}+ \sum_{j=1}^m B_j y_j^{k+1}-c\rangle + \frac{\rho}{2}\|Ax^{k+1} + \sum_{j=1}^m B_j y_{j}^{k+1} - c\|_2^2 \\
		\\
		\geq &f(x^{k+1}) + \sum_{i = 1}^m g_j(y_{j}^{k+1}) - \langle (A^T)^+(v_k - \nabla f(x^k) + \nabla f(x^k) + \frac{G}{\eta}(x^{k+1} - x^k)),Ax^{k+1}+ \sum_{j=1}^m B_j y_j^{k+1}-c\rangle \\
		&+ \frac{\rho}{2}\|Ax^{k+1} + \sum_{j=1}^m B_j y_{j}^{k+1} - c\|_2^2 \\
		\\
		\geq & f(x^{k+1}) + \sum_{i = 1}^m g_j(y_{j}^{k+1}) - \frac{2}{\sigma_{\min}^A \rho} \|v_k - \nabla f(x^k)\|_2^2 - \frac{2}{\sigma_{\min}^A}\|\nabla f(x^k)\|_2^2 - \frac{2\sigma_{\max}^2(G)}{\sigma_{\min}^A \eta^2 \rho}\|x^{k+1} - x^k\|_2^2\\
		& + \frac{\rho}{8}\|Ax^{k+1} + \sum_{j=1}^m B_j y_{j}^{k+1} - c\|_2^2\\
		\\
		\geq &f(x^{k+1}) + \sum_{i = 1}^m g_j(y_{j}^{k+1}) - \frac{2L_F^2}{\sigma_{\min}^A q \rho} \sum_{i = (n_k-1)q}^{k-1}\bE\|x^{i+1} - x^i\|_2^2-\frac{2\cC_1}{\sigma_{\min}^A
		\rho} - \frac{2\ell_1^2\ell_2^2}{\sigma_{\min}^A \rho} - \frac{2\sigma_{\max}^2(G)}{\sigma_{\min}^A \eta^2 \rho}\|x^{k+1} - x^k\|_2^2\\
		\geq & f(x^{k+1}) + \sum_{i = 1}^m g_j(y_{j}^{k+1}) - \frac{2L_F^2}{\sigma_{\min}^A q \rho} \sum_{i = (n_k-1)q}^{k-1}\bE\|x^{i+1} - x^i\|_2^2 - \frac{2\cC_1}{\sigma_{\min}^A
		\rho}- \frac{2\ell_1^2 \ell_2^2}{\sigma_{\min}^A \rho} - \Bigl( \frac{9L_F^2}{\sigma_{\min}^A \rho} + \frac{3\sigma_{\max}^2(G)}{\sigma_{\min}^A \eta^2 \rho}\Bigr)\|x^{k+1} - x^k\|_2^2
	\end{aligned}
\end{equation}

In all, 
\begin{equation}
	R_{k+1} \geq f(x_{k+1}) + \sum_{j=1}^m g_j(x_{k+1}) - \frac{2\ell_1^2 \ell_2^2}{\sigma_{\min}^A \rho}-\frac{2\cC_1}{\sigma_{\min}^A
		\rho} \geq f^* + \sum_{j =1}^m g_j^* -  \frac{2\ell_1^2 \ell_2^2}{\sigma_{\min}^A \rho}-\frac{2\cC_1}{\sigma_{\min}^A
		\rho}
\end{equation}
It follows that the potential function $R_k$ is bounded below. Let's denote the lower bound of $R_k$ is $R^*$. Now we sum up the \eqref{expected_potential} over all the iterates from $0$ to $K$, we will have:
\begin{equation}
	\bE[R_k] - \bE[R_0] \leq - \sum_{i=0}^{K-1}(\Lambda \|x^{i+1} - x^i\|_2^2 + \sigma_{\min}^H \sum_{j = 1}^m \|y_j^i - y_j^{i+1}\|_2^2) + \Bigl(\frac{\cC_1K}{2L_F}+\frac{6 \cC_1 K}{\rho \sigma_{\min}^A}\Bigr)
\end{equation}

Finally, we will have the iteration bound to be:
\begin{equation}
\boxed{
	\frac{1}{K}\sum_{k=0}^K \Bigl(\|x^{k+1} - x^k\|_2^2 + \sum_{j = 1}^m \|y_j^k - y_j^{k+1}\|_2^2\Bigr) \leq \frac{1}{K \gamma}\Bigl(\bE [R_0] - \bE[R_K]\Bigr) + \Bigl(\frac{\cC_1}{2L_F}+\frac{6\cC_1}{\rho \sigma_{\min}^A}\Bigr)
	}
\end{equation}
In which $\gamma = \min(\Lambda, \sigma_{\min}^H)$ and $\Lambda \geq \frac{\sqrt{78}L_F\kappa_G}{4\alpha}$.
\end{proof}

\begin{lemma}[Stationary point convergence]
Suppose the sequence $\{x^k, y^k_{[m]}, z^k\}$ is generated from Algorithm \eqref{algorithm_1}, there exists a constant $\tilde{\nu}$ such that, with $T$ sampling uniformly from $1,...,K$, we will have:
\begin{equation}
\bE \|\dist (0, \partial L(x^T, y_{[m]}^T,z^T))\|_2^2
    \leq \frac{9 \tilde{\nu}}{K\gamma} (R_0 - R_*) +  \frac{9\nu_{\max}}{\gamma}\Bigl(\frac{\cC_1}{2L_F}+\frac{6\cC_1}{\rho \sigma_{\min}^A}\Bigr) + 3\cC_1 + \frac{6\cC_1}{\rho^2 \sigma_{\min}^A}
\end{equation}
\end{lemma}
\begin{proof}
Consider the sequence $\theta_k = \bE[\|x^{k+1} - x^k\|_2^2 + \|x^k - x^{k-1}\|_2^2 + \frac{1}{q}\sum_{i = (n_k-1)q}^k\|x^{i+1} - x^i\|_2^2 + \sum_{j=1}^m\|y_j^k -y_j^{k+1}\|_2^2]$.\\

Consider in the update of $y_i$ component, we will have:
\begin{equation}
\begin{aligned}
&\bE [\dist(0,\partial_{y_i}\cL(x,y_{[m]},z))]_{k+1} =  \bE [\dist(0, \partial g_j(y_j^{k+1}) - B_j^T z^{k+1})]_{k+1}\\
= & \bE \|B_j^T z^k - \rho B_j^T(Ax^k + \sum_{i=1}^j B_j y_j^{k+1} + \sum_{i = j+1}^m B_i y_i^k - c) - H_j(y_j^{k+1} - y_j^k) - B_j^Tz^{k+1}\|_2^2\\
= & \bE\|\rho B_j^T A(x^{k+1} - x^{k+1}) + \rho B_j^T \sum_{i = j+1}^m B_i(y_i^{k+1} - y_i^k) - H_j(y_j^{k+1} - y_j^k)\|_2^2\\
\leq & m \rho^2 \sigma_{\max}^{B_j} \sigma_{\max}^A \bE\|x^{k+1} - x^k\|_2^2 + m \rho^2 \sigma_{\max}^{B_j} \sum_{i = j+1}^m \sigma_{\max}^{B_i}\bE\|y_i^{k+1} - y_i^k\|_2^2 + m \sigma_{\max}^2(H_j)\bE\|y_j^{k+1} - y_j^k\|_2^2\\
\leq & m(\rho^2 \sigma_{\max}^B\sigma_{\max}^A +\rho^2 (\sigma_{\max}^B)^2 + \sigma_{\max}^2(H_j))\theta_k =  \nu_1 \theta_k
\end{aligned}
\end{equation}
In the updating of the $x$-component, we will have:
\begin{equation}
\begin{aligned}
    & \bE [\dist(0,\nabla_x \cL(x,y_{[m]},z))^2]_{k+1} = \bE[\|A^T z^{k+1} - \nabla f(x^{k+1})\|_2^2]\\
    \leq & \bE \|v_k - \nabla f(x^{k+1}) - \frac{G}{\eta}(x^k - x^{k+1}))\|_2^2\\
    \leq & \bE \|v_k - \nabla f(x^k) + \nabla f(x^k) - \nabla f(x^{k+1}) - \frac{G}{\eta}(x^k - x^{k+1}))\|_2^2\\
    \leq & \sum_{i = (n_k-1)q}^{k-1} \frac{3L_F^2}{q} \bE\|x^{i+1} - x^i\|_2^2 + 3\cC_1 + 3(L_F^2 + \frac{\sigma_{\max}^2(G)}{\eta^2})\|x^k - x^{k+1}\|_2^2\\
    \leq & 3(L_F^2 + \frac{\sigma_{\max}^2(G)}{\eta^2})\theta_k + 3\cC_1 = \nu_2 \theta_k + 3\cC_1 
    \end{aligned}
\end{equation}

In the updating of the $z$ component, we will have:
\begin{equation}
\begin{aligned}
& \bE [\dist(0,\nabla_z \cL(x,y_{[m]},z))^2]_{k+1} \\
=& \bE\|Ax^{k+1} + \sum_{j = 1}^m B_j y_j^{k+1} - c\|_2^2\\
= &\frac{1}{\rho^2}\|z^{k+1} - z^k\|_2^2\\
\leq & \frac{6\cC_1}{\rho^2 \sigma_{\min}^A} + \frac{6\cC_2}{\rho^2 \sigma_{\min}^A} \sum_{i = (n_k-1)q}^{k-1}\bE\|x^{i+1} - x^i\|_2^2 + \frac{3\sigma_{\max}^2(G)}{ \rho^2\sigma_{\min}^A\eta^2}\|x^{k+1} - x^k\|_2^2 + (\frac{3\sigma_{\max}^2(G)}{\rho^2\sigma_{\min}^A\eta^2}+\frac{9L^2}{\rho^2 \sigma_{\min}^A})\|x^{k-1}- x^k\|_2^2\\
\leq & (\frac{9L^2}{\rho^2 \sigma_{\min}^A} + \frac{3\sigma_{\max}^2 G }{\rho^2 \sigma_A^2 \eta^2})\theta_k + \frac{6\cC_1}{\rho^2 \sigma_{\min}^A} = \nu_3 \theta_k + \frac{6\cC_1}{\rho^2 \sigma_{\min}^A}
\end{aligned}
\end{equation}
Since we know that:
\[
\begin{aligned}
\sum_{k=1}^{K-1} \theta_k = & \sum_{k=1}^{K-1}\bE[\|x^{k+1} - x^k\|_2^2 + \|x^k - x^{k-1}\|_2^2 + \frac{1}{q}\sum_{i = (n_k-1)q}^k\|x^{i+1} - x^i\|_2^2 + \sum_{j=1}^m\|y_j^k -y_j^{k+1}\|_2^2\\
\leq & 3\Bigl(\sum_{k=1}^{K-1}\bE[\|x^{k+1} - x^k\|_2^2 + \sum_{k=1}^{K-1}\sum_{j=1}^m \bE\|y_j^k - y_j^{k+1}\|_2^2 \Bigr)
\end{aligned}
\]

Now, consider $T$ is chosen uniformly from ${1,2,...,K-1,K}$, we will have the following bound:
\begin{equation}
\begin{aligned}
   &  \bE \|\dist (0, \partial L(x^T, y_{[m]}^T,z^T))\|_2^2\\
    \leq & \frac{3\nu_{\max}}{K}\sum_{k = 1}^K \theta_k  + 3\cC_1 + \frac{6\cC_1}{\rho^2 \sigma_{\min}^A}\\
    \leq &\frac{9\nu_{\max}}{K} \Bigl(\sum_{k=1}^{K-1}\bE[\|x^{k+1} - x^k\|_2^2 + \sum_{k=1}^{K-1}\sum_{j=1}^m \bE\|y_j^k - y_j^{k+1}\|_2^2 \Bigr) +  3\cC_1 + \frac{6\cC_1}{\rho^2 \sigma_{\min}^A}\\
    \leq &\frac{9 \nu_{\max}}{K\gamma} (R_0 - R_*) + \frac{9\nu_{\max}}{\gamma}\Bigl(\frac{\cC_1}{2L_F}+\frac{6\cC_1}{\rho \sigma_{\min}^A}\Bigr) + 3\cC_1 + \frac{6\cC_1}{\rho^2 \sigma_{\min}^A}
    \end{aligned}
\end{equation}

Given $\eta = \frac{2\alpha \sigma _{\min}(G)}{3L}(0 < \alpha < 1)$ and $\Lambda \geq \frac{\sqrt{78}L_F\kappa_G}{4\alpha} $, with $T$ chosen uniformly from ${1,2,...,K-1,K}$, we will have:
\begin{equation}
    \bE\|\dist(0,\partial L (x^R,y^R_{[m]},z^R))\|_2^2 \leq \mathcal{O}(\frac{1}{K}) + \cO{(\cC_1)}
\end{equation}
\end{proof}

\begin{theorem}[Total Sampling complexity] Consider we want to achieve an $\epsilon$-stationary point solution, the total iteration complexity is $\cO(\epsilon^{-2})$. We choose $\cC _1 \sim \cO(\epsilon^{2})$ such that $B_1,B_2,S \sim \cO(\epsilon^{-2})$. We choose $b_1,b_2,s \sim \cO(\epsilon^{-1})$. The size of optimal epoch will be the same order as $b_1,b_2$. After $\cO(\epsilon^{-2})$ iterations, the total sample complexity is $\cO \Bigl(\epsilon^{-3}\Bigr)$
\label{online_theorem}
\end{theorem}
\begin{remark}
	By choosing the right parameter for the ADMM algorithms, we can achieve same complexity as SPIDER and SARAH for online case.
\end{remark}

\bibliographystyle{apalike}
\bibliography{reference.bib}

\begin{thebibliography}{}

\bibitem[Fang et~al., 2018]{fang2018spider}
Fang, C., Li, C.~J., Lin, Z., and Zhang, T. (2018).
\newblock Spider: Near-optimal non-convex optimization via stochastic
  path-integrated differential estimator.
\newblock In {\em Advances in Neural Information Processing Systems}, pages
  689--699.

\bibitem[Lin et~al., 2018]{lin2018improved}
Lin, T., Fan, C., Wang, M., and Jordan, M.~I. (2018).
\newblock Improved oracle complexity for stochastic compositional variance
  reduced gradient.
\newblock {\em arXiv preprint arXiv:1806.00458}.

\bibitem[Yu and Huang, 2017]{yu2017fast}
Yu, Y. and Huang, L. (2017).
\newblock Fast stochastic variance reduced admm for stochastic composition
  optimization.
\newblock {\em arXiv preprint arXiv:1705.04138}.

\bibitem[Zhang and Xiao, 2019]{zhang2019multi}
Zhang, J. and Xiao, L. (2019).
\newblock Multi-level composite stochastic optimization via nested variance
  reduction.
\newblock {\em arXiv preprint arXiv:1908.11468}.

\end{thebibliography}
\end{document}